\documentclass{article}

\usepackage[preprint]{neurips_2026}

\usepackage[utf8]{inputenc} % allow utf-8 input
\usepackage[T1]{fontenc}    % use 8-bit T1 fonts
\usepackage{hyperref}       % hyperlinks
\usepackage{url}            % simple URL typesetting
\usepackage{booktabs}       % professional-quality tables
\usepackage{amsfonts}       % blackboard math symbols
\usepackage{nicefrac}       % compact symbols for 1/2, etc.
\usepackage{microtype}      % microtypography
\usepackage{xcolor}         % colors

\usepackage{placeins}

\usepackage{amsthm}

\newtheorem{example}{Example}

\setcitestyle{square,numbers,comma}

\usepackage{amsmath,amssymb,amsthm,mathtools}
\usepackage{bm}
\usepackage{microtype}
\usepackage{graphicx}
\usepackage{booktabs}
\usepackage{multirow}
\usepackage{array}
\usepackage{enumitem}
\usepackage{hyperref}
\usepackage{url}
\usepackage{algorithm}
\usepackage{algorithmic}
\usepackage{tikz}
\usetikzlibrary{arrows.meta,positioning,fit,calc,decorations.pathreplacing}
\usepackage{pgfplots}
\pgfplotsset{compat=1.18}

\hypersetup{
  colorlinks=true,
  linkcolor=blue,
  citecolor=blue,
  urlcolor=blue
}

\newtheorem{proposition}{Proposition}

\theoremstyle{remark}

\newcommand{\E}{\mathbb{E}}

\newcommand{\eps}{\varepsilon}
\newcommand{\norm}[1]{\left\lVert #1 \right\rVert}

\usepackage{cleveref}
\usepackage{subcaption}

\usepackage{wrapfig}

\usepackage{amsmath,amssymb}
\usetikzlibrary{arrows.meta,calc,positioning}

\definecolor{tradblue}{RGB}{36,92,196}
\definecolor{tradlight}{RGB}{236,242,252}
\definecolor{tradfill}{RGB}{222,232,249}
\definecolor{ourgreen}{RGB}{24,120,80}
\definecolor{ourlight}{RGB}{236,245,240}
\definecolor{ourfill}{RGB}{222,237,228}
\definecolor{textgray}{RGB}{68,73,80}
\definecolor{manifoldblue}{RGB}{193,211,246}
\definecolor{manifoldgreen}{RGB}{196,224,206}

\title{Conservative Flows: A New Paradigm of Generative Models}

\author{%
  Eshed Gal \\
  The University of British Columbia\\
  Vancouver, Canada\\
  \texttt{eshedg@cs.ubc.ca} \\
  \And
  Md Shahriar Rahim Siddiqui \\
  The University of British Columbia\\
  Vancouver, Canada\\
  \texttt{shahriar.siddiqui@ubc.ca} \\
  \And
  Moshe Eliasof \\
  University of Cambridge\\
  Cambridge, United Kingdom\\
  \texttt{me532@cam.ac.uk} \\
  \And
  Eldad Haber \\
  The University of British Columbia\\
  Vancouver, Canada\\
  \texttt{ehaber@eoas.ubc.ca} \\
}

\begin{document}

\maketitle

\begin{abstract}
Modern generative modeling is dominated by transport from a noise prior
to data. We propose an alternative paradigm in which generation is performed
by a discrete stochastic dynamics that leaves the data distribution invariant,
initialized from data-supported states rather than from noise. The framework
can utilize any pretrained flow 
model. We develop two probability-preserving sampling mechanisms, a corrected
Langevin dynamics with a Metropolis adjustment and a predictor-corrector
flow, that operate directly on existing checkpoints. We validate the framework
on a synthetic Swiss-roll target, ImageNet-256 and Oxford Flowers-102,
where our samplers consistently improve over the original
generation procedures.

%On ImageNet-256, this is done directly with publicly
%released flow-matching checkpoints, without modifying any model weights.
\end{abstract}

\section{Introduction}
Generative modeling has been dominated by methods that construct a transport
from a simple reference law, typically, Gaussian noise, to a complex data
distribution \citep{song2021scorebased,ruthotto2021introduction}. In score and flow-based diffusion models
\citep{batzolis2021conditional,song2021maximum, lipman2022flow} this is realized by perturbing
data through a forward noising process and learning score fields for the
resulting family of intermediate distributions; sampling reverses the dynamics
from noise back to data \citep{lebowitz1963dynamical,haussmann1986time,song2019generative,anderson1982reverse, vincent2011connection,dhariwal2021diffusion,lipman2022flow}.
%The reverse-time diffusion viewpoint itself predates modern generative modeling
%\citep{lebowitz1963dynamical,haussmann1986time}; what made recent score-based
%and flow-matching models transformative was not the reverse-time stochastic process alone, but its combination with learning at
%scale \citep{vincent2011connection,dhariwal2021diffusion,lipman2022flow}.

In this paper, we propose a paradigm shift for data generation. Particularly, instead of
learning a global transport from noise to data, we seek to construct a
discrete stochastic dynamics
\begin{equation}
x_0,x_1,x_2,\ldots
\label{eq:intro_chain}
\end{equation}
such that
\begin{equation}
x_0 \sim p,
\qquad
x_j \sim p \quad \text{for all generation steps } j\geq 1,
\label{eq:intro_invariance_goal}
\end{equation}
where \(p\) denotes the data distribution. In other words, our approach states that rather than
learning how to move \emph{toward} the data distribution from an initial
prior, we learn how to move \emph{within} the data distribution while
preserving it at every step. As such, our approach shifts the goal of generative models from
\emph{transport} to \emph{invariance}: rather than transforming noise into
data, the sampler maintains the data distribution under a learned
probability-conserving dynamics. An illustration of our paradigm is illustrated in Figure~\ref{fig:teaser}.

% \begin{figure}[t]
%     \centering
%     \begin{tabular}{ccccc}
% \includegraphics[width=0.15\linewidth]{Figures/traj0.png} &
% \includegraphics[width=0.15\linewidth]{Figures/traj100.png} &
% \includegraphics[width=0.15\linewidth]{Figures/traj500.png} &
% \includegraphics[width=0.15\linewidth]{Figures/traj1500.png} &
% \includegraphics[width=0.15\linewidth]{Figures/traj2500.png} \\
% $t=0$ & $t=100$ & $t=500$ & $t=1500$ & $t=2500$
% \end{tabular}
%     \caption{Conservative sampling on a Swiss-roll-like target. Starting from
%     samples \(x_0\sim p\), the corrected dynamics moves within the target
%     distribution while preserving its law.}
%     \label{fig:swiss_role}
% \end{figure}

\begin{figure}[t]
    \centering
    \resizebox{\textwidth}{!}{%
\begin{tikzpicture}[x=1cm,y=1cm,>=Latex,
    panel/.style={rounded corners=10pt, line width=0.95pt, fill=white},
    titleblue/.style={font=\bfseries\fontsize{17}{19}\selectfont, text=tradblue},
    titlegreen/.style={font=\bfseries\fontsize{17}{19}\selectfont, text=ourgreen},
    subtitle/.style={font=\fontsize{11}{13}\selectfont, text=textgray, align=center},
    bodytext/.style={font=\fontsize{10.6}{12.3}\selectfont, text=textgray, align=center},
    labelblue/.style={font=\bfseries\fontsize{11.2}{12.5}\selectfont, text=tradblue, align=center},
    labelgreen/.style={font=\bfseries\fontsize{11.2}{12.5}\selectfont, text=ourgreen, align=center},
    auxblue/.style={font=\fontsize{9.4}{10.8}\selectfont, text=tradblue, align=center},
    auxgreen/.style={font=\fontsize{9.4}{10.8}\selectfont, text=ourgreen, align=center},
    ambientlabel/.style={font=\fontsize{11}{12.5}\selectfont, align=center},
    badge/.style={rounded corners=4pt, line width=0.65pt, fill=white}
]

% Increased \W slightly to give inner boxes more breathing room
% Increased \gap to slightly separate the left and right panels
\def\W{17.2}
\def\H{7.4}
\def\gap{0.6} 

\coordinate (L) at (0,0);
\coordinate (R) at (\W+\gap,0);

% Strict Bounding Box
\useasboundingbox (0,0) rectangle (\W+\W+\gap, \H);

% Panel outlines
\draw[panel, draw=tradblue] (L) rectangle ++(\W,\H);
\draw[panel, draw=ourgreen] (R) rectangle ++(\W,\H);

% Titles
\node[titleblue, anchor=north] at ($(L)+(\W/2,\H-0.28)$) {Existing Generative Approaches};
\node[subtitle, anchor=north] at ($(L)+(\W/2,\H-1.05)$) {Generate by transporting from a simple prior to the data distribution $p$.};

\node[titlegreen, anchor=north] at ($(R)+(\W/2,\H-0.28)$) {Our Approach: Probability-Preserving Generation};
\node[subtitle, anchor=north] at ($(R)+(\W/2,\H-1.05)$) {Start from data-supported states and evolve while preserving the probability law.};

% Ambient space boxes (Made wider: 16.6 instead of 15.8)
\fill[tradlight] ($(L)+(0.3,1.15)$) rectangle ++(16.6,4.5);
\draw[rounded corners=3pt, draw=tradblue!22, line width=0.6pt] ($(L)+(0.3,1.15)$) rectangle ++(16.6,4.5);

\fill[ourlight] ($(R)+(0.3,1.15)$) rectangle ++(16.6,4.5);
\draw[rounded corners=3pt, draw=ourgreen!22, line width=0.6pt] ($(R)+(0.3,1.15)$) rectangle ++(16.6,4.5);

% ---------------------------------------------------------
% LEFT PANEL
% ---------------------------------------------------------
% Left panel: Gaussian support shown as circle
\fill[tradblue!6] ($(L)+(2.75,3.55)$) circle (1.65);
\draw[draw=tradblue!30, line width=0.9pt] ($(L)+(2.75,3.55)$) circle (1.65);
\node[labelblue, anchor=north] at ($(L)+(2.75,1.15)$) {Gaussian support};
\node[bodytext, anchor=north] at ($(L)+(2.75,0.55)$) {$x_0 \sim \mathcal{N}(0,I)$};

\foreach \x/\y in {1.75/4.35,2.25/3.85,2.85/4.45,3.5/4.0,3.85/3.35,3.25/2.75,2.5/2.6,1.95/3.1,3.3/3.35,2.3/4.7}
    \fill[tradblue] ($(L)+(\x,\y)$) circle (0.08);

% Middle transport box and arrows
\draw[tradblue, line width=1.05pt, -{Latex[length=3mm]}] ($(L)+(4.75,3.55)$) -- ($(L)+(5.85,3.55)$);
\draw[rounded corners=5pt, line width=0.95pt, draw=tradblue, inner sep=2pt, fill=white] ($(L)+(5.9,2.65)$) rectangle ++(4.15,1.8);
\node[bodytext, text=tradblue, font=\bfseries\fontsize{11.2}{12.4}\selectfont] at ($(L)+(7.98,4.0)$) {Learned transport};
\node[bodytext] at ($(L)+(7.98,3.45)$) {reverse diffusion /};
\node[bodytext] at ($(L)+(7.98,2.95)$) {flow matching};
\draw[tradblue, line width=1.05pt, -{Latex[length=3mm]}] ($(L)+(10.1,3.55)$) -- ($(L)+(11.3,3.55)$);

% Left Panel Manifold
\draw[tradblue!32, line width=24pt, line cap=round, shift={(L)}]
    plot[domain=12.5:16.5, samples=40, smooth] (\x, {4.0 + 0.6*sin((\x-12.5)*360/4.0)});
\draw[manifoldblue!48, line width=22pt, line cap=round, shift={(L)}]
    plot[domain=12.5:16.5, samples=40, smooth] (\x, {4.0 + 0.6*sin((\x-12.5)*360/4.0)});

% Mathematically centered dots tracking the sine wave
\fill[tradblue] ($(L)+(12.90,4.35)$) circle (0.09);
\fill[tradblue] ($(L)+(13.50,4.60)$) circle (0.09);
\fill[tradblue] ($(L)+(14.10,4.35)$) circle (0.09);
\fill[tradblue] ($(L)+(14.80,3.73)$) circle (0.09);
\fill[tradblue] ($(L)+(15.50,3.40)$) circle (0.09);
\fill[tradblue] ($(L)+(16.10,3.65)$) circle (0.09);

\node[labelblue, anchor=north] at ($(L)+(14.65,1.15)$) {Data distribution};
\node[bodytext, anchor=north] at ($(L)+(14.65,0.55)$) {$x \sim p$};

% ---------------------------------------------------------
% RIGHT PANEL 
% ---------------------------------------------------------
% Right Panel Manifold (Centered properly in the new box width)
\draw[ourgreen!33, line width=24pt, line cap=round, shift={(R)}]
    plot[domain=1.1:16.1, samples=60, smooth] (\x, {4.0 + 0.6*sin((\x-1.1)*360/7.5)});
\draw[manifoldgreen!72, line width=22pt, line cap=round, shift={(R)}]
    plot[domain=1.1:16.1, samples=60, smooth] (\x, {4.0 + 0.6*sin((\x-1.1)*360/7.5)});

% Defined perfectly sized nodes to prevent dashed arrows overlapping the dots
\node[circle, fill=ourgreen, inner sep=0pt, minimum size=0.2cm] (r1) at ($(R)+(1.80,4.33)$) {};
\node[circle, fill=ourgreen, inner sep=0pt, minimum size=0.2cm] (r2) at ($(R)+(3.60,4.52)$) {};
\node[circle, fill=ourgreen, inner sep=0pt, minimum size=0.2cm] (r3) at ($(R)+(5.40,3.73)$) {};
\node[circle, fill=ourgreen, inner sep=0pt, minimum size=0.2cm] (r4) at ($(R)+(7.40,3.49)$) {};
\node[circle, fill=ourgreen, inner sep=0pt, minimum size=0.2cm] (r5) at ($(R)+(9.60,4.45)$) {};
\node[circle, fill=ourgreen, inner sep=0pt, minimum size=0.2cm] (r6) at ($(R)+(11.10,4.52)$) {};
\node[circle, fill=ourgreen, inner sep=0pt, minimum size=0.2cm] (r7) at ($(R)+(13.10,3.65)$) {};
\node[circle, fill=ourgreen, inner sep=0pt, minimum size=0.2cm] (r8) at ($(R)+(15.10,3.55)$) {};

% Automatically calculated dynamic hops
\begin{scope}[every path/.style={ourgreen!85!black, dashed, line width=0.95pt, -{Latex[length=2.6mm]}}]
    \draw (r1) to[bend left=25] (r2);
    \draw (r2) to[bend left=25] (r3);
    \draw (r3) to[bend left=25] (r4);
    \draw (r4) to[bend left=25] (r5);
    \draw (r5) to[bend left=25] (r6);
    \draw (r6) to[bend left=25] (r7);
    \draw (r7) to[bend left=25] (r8);
\end{scope}

% Text positioned completely underneath the wave for zero overlap
\node[labelgreen, anchor=south] at ($(R)+(8.6,2.6)$) {Data manifold $\mathcal{M}$};
\node[auxgreen, anchor=north] at ($(R)+(8.6,2.55)$) {learned distribution manifold};

% Aligned right panel text to match the Left Panel's anchor styling and coordinates
\node[labelgreen, anchor=north] at ($(R)+(2.75,1.15)$) {Sample from probability};
\node[bodytext, anchor=north] at ($(R)+(2.75,0.55)$) {$x_0 \sim p$};

\node[labelgreen, anchor=north] at ($(R)+(14.65,1.15)$) {Evolve };
\node[bodytext, anchor=north] at ($(R)+(14.65,0.55)$) {preserving $p$};

\end{tikzpicture}
    }
\caption{From transport to invariance. Existing generative models produce samples by transporting Gaussian noise to the data distribution. Our approach instead starts from data-supported states and evolves them through probability-preserving dynamics. By keeping the process on high-probability regions of the data manifold, this enables faithful sample variation, improved data fidelity, and efficient reuse of pretrained generative models.}   %\caption{A comparison of existing generative model in contrast to our approach. While existing models flow from a Gaussian to the distribution, our approach starts at the distribution and flows within the data manifold enabling better data fidelity.}
    \label{fig:teaser}
\end{figure}
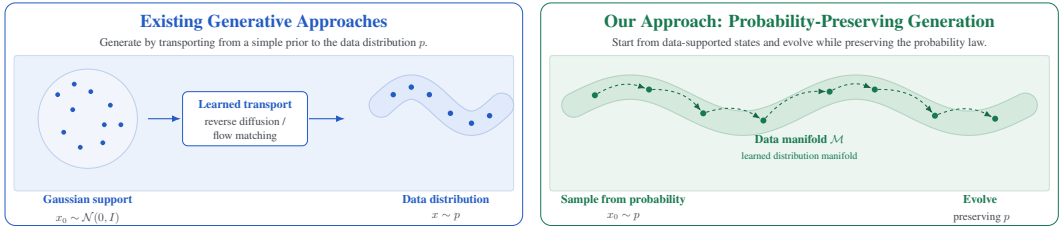

This viewpoint is particularly natural in modern augmented retrieval systems
\citep{lewis2020retrieval} and latent generative pipelines
\citep{rombach2022high}. In such systems, generation does not begin
from an information-less state. Instead, one has access to existing valid examples,
nearest neighbors, memory entries, or latent codes that already lie near the
target distribution of interest \citep{blattmann2022retrieval}. In such
settings, starting from pure noise is not the only viable option, and may
not even be the most natural one. If one can start the generation process from data-supported
states, then a probability-preserving dynamics becomes a compelling
alternative to the common approach of noise-to-data generation.

At an infinitesimal level, Langevin dynamics, which is at the heart of modern generative models \cite{ho2020denoising,song2019generative}, provides precisely this
intuition \citep{bussi2007accurate}. That is, if \(s(x)=\nabla \log p(x)\) is the
score of the target density, then under standard regularity assumptions, the
continuous-time diffusion, that reads:
\begin{equation}
dX_t = s(X_t)\,dt + \sqrt{2}\,dW_t
\label{eq:intro_langevin_sde},
\end{equation}
preserves \(p\) as an invariant measure \citep{roberts1996exponential}.
Thus, in continuous time, score-based dynamics already define a
probability-preserving evolution. The difficulty arises \emph{after} discretization.
That is, the finite-step Langevin update defined by 
\begin{equation}
x_{j+1} = x_j + h\,s(x_j) + \sqrt{2h}\,\xi_j,
\qquad
\xi_j \sim \mathcal N(0,I),
\label{eq:intro_ula}
\end{equation}
has the correct infinitesimal behavior, but it does not preserve \(p\)
exactly for finite \(h\). This discrepancy is negligible for sufficiently small
steps, but it changes the distribution, and therefore breaks the central
requirement in Equation \eqref{eq:intro_invariance_goal} \citep{vempala2019rapid}. Our starting point is that
a generative model based on stationary dynamics, i.e., probability-preserving, as envisioned in our paradigm, should address this issue
directly.

We therefore develop two probability-preserving mechanisms that start from
samples already drawn from the density of interest. The \emph{first} is a corrected
Langevin dynamics based on a conservative score, or equivalently a denoiser.
In its score-based form, the correction restores invariance of \(p\); in its
denoiser-based form, it naturally acts on a Gaussian-smoothed law
\(p_\sigma\) and therefore preserves \(p_\sigma\). The \emph{second} is a
predictor--corrector flow construction: starting from \(x\sim p\), one first
adds a controlled amount of noise to move the sample to an intermediate
bridge distribution, and then uses a learned flow or direct solution map to
return it to \(p\). Thus we   present two distinct but related ways
to probe a probability density when initialization already lies on, or near,
that density. Notably, in both cases, the required score, denoiser, or flow, can be
obtained directly from a pretrained flow-matching model, so neither
mechanism requires training a new generative model from scratch.

The role of retrieval is then conceptually simple. Generation is performed
by initializing the sampler from retrieved data points or retrieved latent
codes, and then stepping forward with one of these probability-preserving
mechanisms. In the ideal case there is no burn-in, since the chain starts
from the target law itself; more generally, retrieved initializations place
the sampler in a high-probability region of the relevant data manifold.
This makes the framework especially attractive in latent spaces, where
retrieval is natural, geometry is smoother, and local stochastic exploration
can produce meaningful variation at lower computational cost
\citep{bengio2013representation}. We discuss related works in \Cref{app:related_work}.

The resulting approach differs from both standard diffusion models and
classical energy-based modeling. Diffusion models learn a reverse transport
from noise through a family of perturbed distributions
\citep{song2019generative,song2021scorebased}. This typically requires learning a
time-dependent score field \(s(x,t)\) across many noise levels. In contrast,
our first method does not require such a family of scores: because
generation is initialized from data-supported states and proceeds by local
probability-preserving moves, it is sufficient to learn a single conservative
score field, or one associated with a small fixed noise level. Energy-based
models, by contrast, typically learn a scalar energy and then sample by
approximate Langevin dynamics, often from noise or from a replay mechanism
\citep{du2019implicit}. Our approach instead begins from data-supported
states and then either uses a corrected score/denoiser-based move or a
predictor--corrector flow move. \textbf{A key practical advantage} of our approach is the ability to use pretrained flow matching models. Similar to \citep{keegan2026manifold},  our framework can be applied directly on top of such models,
inheriting their manifold-aware dynamics \emph{with no further training}. 
Recent work has also emphasized the importance
of conservation in score-based models and explored Metropolis-type
corrections based on line integration of the score
\citep{chao2023investigating,sjoberg2023mcmc}; our contribution is to combine
these ideas into a generative framework whose organizing principle is not
reverse noising, but stationary probing of the data distribution itself.

We view this as a new paradigm for generative modeling. The central object
is not a transport map from a simple prior, but a stochastic dynamics that
preserves the relevant law. The central initialization is not Gaussian
noise, but retrieved or observed data-supported states. And the central
learning problem is not to model a sequence of noisy marginals, but either
to learn a conservative score or denoiser together with a discrete
correction, or to learn a local flow that maps partially noised samples
back to the target law. In this sense, the proposed framework replaces
noise-to-data generation by retrieval-seeded stationary generation. In the
ideal setting of an exact conservative score and exact line integration,
the first mechanism preserves the appropriate target law exactly; in the
ideal setting of an exact bridge flow, the second mechanism preserves \(p\)
exactly. In practice, the quality of these approximations depends on the
learned score, denoiser, or flow model, and on the numerical approximation
used in the correction step.

\section{Method}
\label{sec:method}

We seek generative dynamics that leave a target law invariant. Throughout, \(p\) denotes the target data density on \(\mathbb R^d\), and we assume access to samples \(x\sim p\). We will use two related invariant laws. The score-based and denoiser-based Metropolis constructions operate on a Gaussian-smoothed distribution
\[
p_\sigma = p * \mathcal N(0,\sigma^2 I),
\]
and therefore leave \(p_\sigma\) invariant. In contrast, the predictor--corrector flow construction described later returns partially noised samples directly to the target law \(p\), and therefore leaves \(p\) invariant in the ideal case. Keeping this distinction explicit is important for both the theory and the experiments.

\subsection{On the necessity of Langevin Dynamics Correction}
\label{subsec:langevin}

A natural starting point is overdamped Langevin dynamics
\begin{equation}
dX_t = s(X_t)\,dt + \sqrt{2}\,dW_t,
\qquad
s(x)=\nabla \log p(x),
\label{eq:langevin_sde_method_short}
\end{equation}
which preserves \(p\) in continuous time under standard regularity assumptions \citep{roberts1996exponential}. Thus, at the infinitesimal level, Langevin dynamics is probability-conserving.

The difficulty appears after discretization. The Euler--Maruyama step is
\begin{equation}
y = x + h\,s(x) + \sqrt{2h}\,z,
\qquad
z\sim\mathcal N(0,I),
\label{eq:ula_step_method_short}
\end{equation}
which is the \textbf{U}nadjusted \textbf{L}angevin \textbf{A}lgorithm (ULA) \citep{durmus2017nonasymptotic}. Although Equation \eqref{eq:ula_step_method_short} has the correct local geometry, it does not preserve the target law exactly for finite step size.

A Gaussian example makes this explicit. If
\[
p(x)=\mathcal N(0,I),
\qquad
s(x)=-x,
\]
then Equation \eqref{eq:ula_step_method_short} becomes
\[
y=(1-h)x+\sqrt{2h}\,z.
\]
If \(x\sim\mathcal N(0,I)\), then
\[
y\sim \mathcal N(0,(1+h^2)I),
\]
which differs from \(p\) for every \(h>0\). Even in this simplest case, the discretized Langevin step is not probability-invariant.

This is the basic issue: Langevin dynamics is correct in continuous time, but its explicit discretization changes the law. The first part of our method corrects this finite-step error.

\subsection{Corrected Langevin Dynamics and a Denoiser formulation}
\label{subsec:correction}

Starting from the proposal, Equation \eqref{eq:ula_step_method_short}, define the deterministic part of the update by
\begin{equation}
D_h(x):=x+h\,s_\theta(x),
\label{eq:Dh_def}
\end{equation}
so that the proposal can be written as
\begin{equation}
y = D_h(x) + \sqrt{2h}\,z,
\qquad
z\sim\mathcal N(0,I),
\label{eq:Dh_proposal}
\end{equation}
with proposal density
\begin{equation}
q_h(y\mid x)=\mathcal N\!\left(y;\,D_h(x),\,2h\,I\right).
\label{eq:Dh_q}
\end{equation}

When \(p\) is known up to a normalizing constant, the standard correction is \textbf{M}etropolis-\textbf{A}djusted \textbf{L}angevin \textbf{A}lgorithm (MALA) \citep{roberts1996exponential,roberts1998optimal}. Its Metropolis--Hastings log-ratio is
\begin{equation}
\log r(x,y)
=
\log p(y)-\log p(x)
+\log q_h(x\mid y)-\log q_h(y\mid x),
\label{eq:log_r_rewrite}
\end{equation}
and the proposal is accepted with probability
\begin{equation}
\alpha(x,y)=\min\{1,\exp(\log r(x,y))\}.
\label{eq:alpha_rewrite}
\end{equation}
This correction restores invariance of \(p\).

In generative modeling, however, \(p\) is not available explicitly. To resolve this problem, we assume access to a learned score field \(s_\theta\). If this field is conservative, then the missing log-density difference can be recovered by line integration:
\begin{equation}
\Delta_{D_h}(x,y)
:=
\frac{1}{h}
\int_0^1
\Bigl\langle
D_h\bigl(x+t(y-x)\bigr)-\bigl(x+t(y-x)\bigr),\,y-x
\Bigr\rangle\,dt.
\label{eq:Dh_delta}
\end{equation}
The corrected acceptance ratio is therefore
\begin{equation}
\log r_{D_h}(x,y)
=
\Delta_{D_h}(x,y)
+\log q_h(x\mid y)-\log q_h(y\mid x),
\label{eq:Dh_acceptance}
\end{equation}
and the proposal is accepted with probability
\begin{equation}
\alpha_{D_h}(x,y)
=
\min\!\left\{1,\exp\bigl(\log r_{D_h}(x,y)\bigr)\right\}.
\label{eq:Dh_acceptance_probability}
\end{equation}
When \(D_h(x)=x+h\nabla\log p(x)\) and the integral Equation \eqref{eq:Dh_delta} is evaluated exactly, this reduces to dMALA and preserves \(p\).

This notation becomes especially useful when \(D_h\) is obtained from a denoiser. Let
$p_\sigma$ 
be the Gaussian-smoothed density, and let \(D_\sigma\) be a denoiser trained to predict the clean sample from a noisy observation. Tweedie's identity gives
\begin{equation}
D_\sigma(y)=y+\sigma^2 \nabla \log p_\sigma(y).
\label{eq:tweedie_denoiser_rewrite}
\end{equation}
Thus \(D_\sigma\) is exactly of the form Equation \eqref{eq:Dh_def} with
%\begin{equation}
$h=\sigma^2$ and $s_\sigma(y)=\nabla\log p_\sigma(y).$
%\label{eq:h_sigma_relation}
%\end{equation}
The proposal becomes
\begin{equation}
y = D_\sigma(x) + \sqrt{2}\,\sigma\,z,
\qquad
z\sim\mathcal N(0,I),
\label{eq:denoiser_proposal_rewrite}
\end{equation}
with proposal density
\begin{equation}
q_\sigma(y\mid x)=\mathcal N\!\left(y;\,D_\sigma(x),\,2\sigma^2 I\right).
\label{eq:denoiser_q_rewrite}
\end{equation}
The corresponding line integral is
\begin{equation}
\Delta_{D_\sigma}(x,y)
=
\frac{1}{\sigma^2}
\int_0^1
\Bigl\langle
D_\sigma\bigl(x+t(y-x)\bigr)-\bigl(x+t(y-x)\bigr),\,y-x
\Bigr\rangle\,dt,
\label{eq:denoiser_delta_rewrite}
\end{equation}
so the exact correction leaves the smoothed density \(p_\sigma\) invariant. This point is important: the denoiser-based Metropolis construction preserves \(p_\sigma\), not the original target law \(p\).

A further simplification is obtained by approximating Equation \eqref{eq:denoiser_delta_rewrite} with the trapezoidal rule. Define the residual
\begin{equation}
r_\sigma(x):=D_\sigma(x)-x.
\label{eq:denoiser_residual_rewrite}
\end{equation}
Then, 
\begin{equation}
\Delta_{D_\sigma}(x,y)
\approx
\frac{1}{2\sigma^2}\,\bigl\langle r_\sigma(x)+r_\sigma(y),\,y-x\bigr\rangle.
\label{eq:denoiser_trap_rewrite}
\end{equation}
Combining this with the Gaussian proposal ratio, the cross terms cancel and the approximate log-acceptance ratio becomes
\begin{equation}
\log r_{D_\sigma}(x,y)
\approx
\frac{1}{4\sigma^2}
\Bigl(
\norm{D_\sigma(x)-x}^2
-
\norm{D_\sigma(y)-y}^2
\Bigr).
\label{eq:trap_acceptance_simplified_rewrite}
\end{equation}

This yields a particularly simple denoiser-based sampler requiring only two evaluations of \(D_\sigma\)
per step. We refer to this sampler as \textbf{d}enoising MALA, or \textbf{dMALA}, in short, because it reduces to the standard Metropolis-adjusted Langevin algorithm in the conservative-score, exact-line-integration limit.

Training the denoiser is straightforward. For \(\sigma\) fixed or sampled from a small interval close to $0$, one minimizes
\begin{equation}
\mathcal L_{\mathrm{denoise}}
=
\E_{x,\eps,\sigma}
\bigl\|
D_\sigma(x+\sigma\eps)-x
\bigr\|^2,
\qquad
\eps\sim\mathcal N(0,I).
\label{eq:denoiser_training_objective_rewrite}
\end{equation}

In order to encourage a conservative score field, we use the classical fact from vector calculus \cite{marsden_tromba_vector_calculus} that if \(s:\Omega\subset\mathbb R^d\to\mathbb R^d\) is \(C^1\), \(\Omega\) is simply connected, and its Jacobian \(\nabla s(x)\) is symmetric for every \(x\in\Omega\), then \(s\) is a gradient field; that is, there exists a scalar potential \(\phi\) such that
\[
s(x)=\nabla \phi(x).
\]
Since
\begin{equation}
J_{D_\sigma}(x)=I+\sigma^2 J_{s_\sigma}(x),
\label{eq:denoiser_score_jacobian_relation_rewrite}
\end{equation}
the denoiser Jacobian \(J_{D_\sigma}(x)\) is symmetric if and only if the score Jacobian \(J_{s_\sigma}(x)\) is symmetric. Therefore, to encourage a conservative score field, it is sufficient to regularize the denoiser Jacobian. We do so through
\begin{equation}
\mathcal L_{\mathrm{sym}}^{D}(x)
=
\mathbb E_v
\bigl\|
J_{D_\sigma}(x)v - J_{D_\sigma}(x)^\top v
\bigr\|^2,
\label{eq:denoiser_symmetry_penalty_rewrite}
\end{equation}
which is numerically more stable than regularizing the score directly.

A practical reason to work with a denoiser is that strong pretrained flow-matching models are often already available. Rather than training a separate score network, one can extract a denoiser directly from the bridge on which the flow model was trained. This keeps the method internally consistent: the same interpolation used in flow training defines the denoiser used at sampling time. The corresponding construction is described in \Cref{app:flow_to_denoiser}.

\begin{example}[ULA versus corrected Langevin - MALA - on a Swiss-roll-like target]
To illustrate the need for correction, we consider a two-dimensional Swiss-roll-like target density and initialize particles from the target distribution. We then evolve them using either the unadjusted Langevin algorithm (ULA) or the corrected dynamics above. Since both methods start from \(x_0\sim p\), an ideal probability-preserving process should satisfy \(x_k\sim p\) for all \(k\). In practice, ULA drifts away from the target, while the corrected dynamics remains close to the original distribution.
\end{example}

% \begin{figure}[t]
%     \centering
%     \includegraphics[width=0.65\linewidth]{Figures/ULA_MALA_1.png}
%     \caption{Comparison between unadjusted Langevin dynamics (ULA) and the corrected dynamics (MALA) on a Swiss-roll-like distribution. ULA drifts away from the target, while the corrected dynamics preserves it.}
%     \label{fig:ULAMALA}
% \end{figure}

\begin{wrapfigure}{r}{0.5\textwidth}
    \centering
    \includegraphics[width=\linewidth]{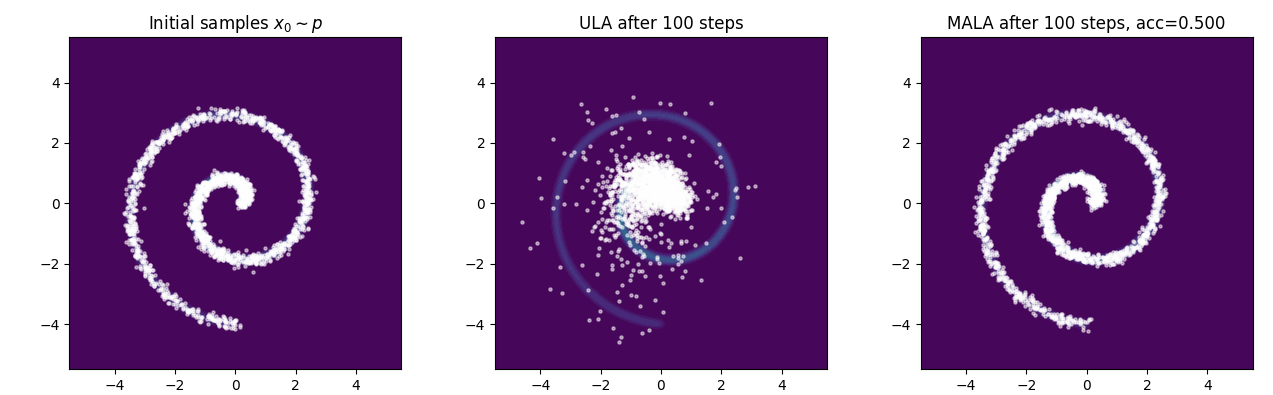}
    \caption{Comparison between unadjusted Langevin dynamics (ULA) and the corrected dynamics (dMALA) on a Swiss-roll-like distribution. ULA drifts away from the target, while the corrected dynamics preserves it.}
    \label{fig:ULAMALA}
\end{wrapfigure}

\subsection{A predictor--corrector approach for probability-invariant flow}
\label{subsec:predictor_corrector_flow}

The previous subsection used a local stochastic proposal together with a Metropolis correction. In the denoiser formulation, this yields a Markov chain \citep{chung1967markov} whose invariant law is the smoothed distribution \(p_\sigma\). We now describe a different construction that targets the original law \(p\) directly.

The corrected Langevin dynamics above uses Langevin as a local predictor and a Metropolis step as a correctness mechanism: if the proposal is inaccurate, it is simply rejected. This is effective when the local geometry is mild, but for highly curved or irregular manifolds the predictor can be poor, leading to low acceptance rates and consequently very slow exploration. In such settings, insisting that every proposal remain close to the high-probability region at all times may be unnecessarily restrictive. The flow-based approach considered here follows a different principle. Instead of proposing a local step and rejecting it when it leaves the manifold, we first allow the sample to move a controlled distance away from the target distribution by adding a small amount of noise, and then use a learned flow to transport it back to the data distribution. This predictor--corrector mechanism can make substantial moves even when the geometry of the manifold is highly nonlinear, while still preserving the target law in the ideal case.

Let \(p_1=p\) denote the target law, let \(p_0=\mathcal N(0,I)\), and consider the linear bridge
\begin{equation}
x_t=(1-t)z+t x,
\qquad
x\sim p_1,\quad z\sim p_0.
\label{eq:pc_linear_bridge_short}
\end{equation}
We denote by \(p_t\) the law of \(x_t\). Thus \(p_t\) interpolates between noise at \(t=0\) and data at \(t=1\).

Suppose now that we have trained a flow model associated with this bridge. This may be represented either by a velocity field
\begin{equation}
\dot x = v_\theta(x,t),
\qquad
x(\tau)=x_\tau,
\qquad
t\in[\tau,1],
\label{eq:pc_velocity_field_short}
\end{equation}
with flow map \(\Phi_{\theta,t\to s}\), or by a direct solution map
\begin{equation}
f_\theta(x_t,t,s)\approx \Phi_{t\to s}(x_t).
\label{eq:pc_solution_map_short}
\end{equation}
In our experiments, the latter is particularly convenient because it avoids numerical ODE integration at sampling time.

Starting from \(x\sim p_1\), choose \(\tau\in(0,1)\), typically close to \(1\), and draw
\begin{equation}
\hat x=(1-\tau)z+\tau x,
\qquad
z\sim\mathcal N(0,I).
\label{eq:pc_predictor_short}
\end{equation}
Since \(\hat x\sim p_\tau\), this predictor step moves the sample slightly away from the target distribution while remaining on the bridge. We then apply the learned flow from time \(\tau\) back to time \(1\):
\begin{equation}
x^+ = \Phi_{\theta,\tau\to 1}(\hat x), \approx
x^+ = f_\theta(\hat x,\tau,1).
\label{eq:pc_corrector_direct_short}
\end{equation}

Under an exact flow model, this move is probability-invariant.

\begin{proposition}[Probability invariance of the predictor--corrector kernel]

Assume that for every \(0\le t\le s\le 1\),
\begin{equation}
(\Phi_{t\to s})_\# p_t = p_s.
\label{eq:pc_pushforward_property_short}
\end{equation}
Let \(x\sim p_1\), draw \(z\sim p_0\) independently, form \(\hat x\) by Equation \eqref{eq:pc_predictor_short}, and set \(x^+=\Phi_{\tau\to 1}(\hat x)\). Then
$
x^+\sim p_1=p.$
\end{proposition}

% \begin{proof}
% By construction, \(\hat x\sim p_\tau\). Applying the exact flow map from \(\tau\) to \(1\) gives
% \[
% x^+=\Phi_{\tau\to 1}(\hat x)\sim (\Phi_{\tau\to 1})_\# p_\tau = p_1.
% \]
% \end{proof}
\begin{proof}
By construction, \(\hat x\sim p_\tau\). Applying the exact flow map from \(\tau\) to \(1\) gives
\[
x^+=\Phi_{\tau\to 1}(\hat x)\sim (\Phi_{\tau\to 1})_\# p_\tau = p_1. \qedhere
\]
\end{proof}
This yields a probability-preserving Markov kernel
\begin{equation}
K_\tau(x,\cdot):
\qquad
x \mapsto x^+ = \Phi_{\theta,\tau\to 1}\bigl((1-\tau)z+\tau x\bigr),
\qquad
z\sim\mathcal N(0,I).
\label{eq:pc_kernel_short}
\end{equation}
For \(\tau\) close to \(1\), this is a local move. Smaller values of \(\tau\) give stronger refreshment and potentially better mixing, but rely more heavily on the accuracy of the learned flow. A practical consequence is that the model only needs to be accurate near the data end of the bridge. Since the sampler uses \(\tau\) close to \(1\), it is sufficient to train on times
$
t\in[\tau_{\min},1]
\label{eq:pc_local_training_interval_short}
$
rather than on the full interval \([0,1]\).

Conceptually, this construction replaces Metropolis rejection by a learned return map to the target distribution. The predictor adds a controlled amount of noise, and the corrector follows the learned probability flow back to \(p\). This provides a second route to invariance, now based on flow matching rather than corrected Langevin dynamics.

%%%%%

\section{Experiments}
\label{sec:exp}

We design experiments to address the following research questions: \textbf{(i)} Does the corrected dynamics preserve a target law
        in practice, where unadjusted Langevin discretization drifts
        away from it?; \textbf{(ii)} Do the two probability-preserving samplers improve
        over the baseline samplers of their pretrained backbones on
        ImageNet-256?; 
\textbf{(iii)} How do the two samplers compare to each other,
        both in long-run exploration on a known density and across
        operating points on a high-dimensional benchmark?; and
\textbf{(iv)} How do the samplers behave under varying chain
        length and computational budget?

\textbf{Baselines.}
For ImageNet-256, we compare against the official 1-NFE solution-operator
sampler of SoFlow-XL/2-cond \citep{luo2025soflow} and no-CFG 400-step
Euler sampling using SiT-REPA-XL \citep{yu2024representation}, and we operate within the latent space of the Stable Diffusion VAE \citep{rombach2022high}, consistent with the pretrained baselines. For the
synthetic Swiss-roll experiment, the relevant baseline is the unadjusted
Langevin algorithm (ULA) \citep{durmus2017nonasymptotic}. All
ImageNet-256 metrics are computed on 50{,}000 generated samples against
the ADM ImageNet-256 reference batch. We present further experimental results on the Oxford-Flowers-102 dataset in Appendix~\ref{app:flowers_results}

\subsection{Probability preservation on a synthetic target}
\label{subsec:swissroll_experiments}

We compare the two probability-preserving mechanisms of \Cref{sec:method} on a two-dimensional Swiss-roll-like target density given by a Gaussian mixture supported along a spiral. This problem is useful since the target density is known exactly, so we can evaluate negative log-likelihood under the true law, while the geometry is sufficiently nonlinear to make long-term mixing visually and quantitatively informative. In all experiments below, both methods use the same pretrained SoFlow model \citep{luo2025soflow}. dMALA is implemented through the denoiser extracted from that flow model, while the predictor--corrector method uses the direct solution map itself.

\begin{wrapfigure}{r}{0.5\textwidth} % 0.3 makes it much narrower. You can also use {5cm}
    \vspace{-2em}
    \centering
    \includegraphics[width=\linewidth]{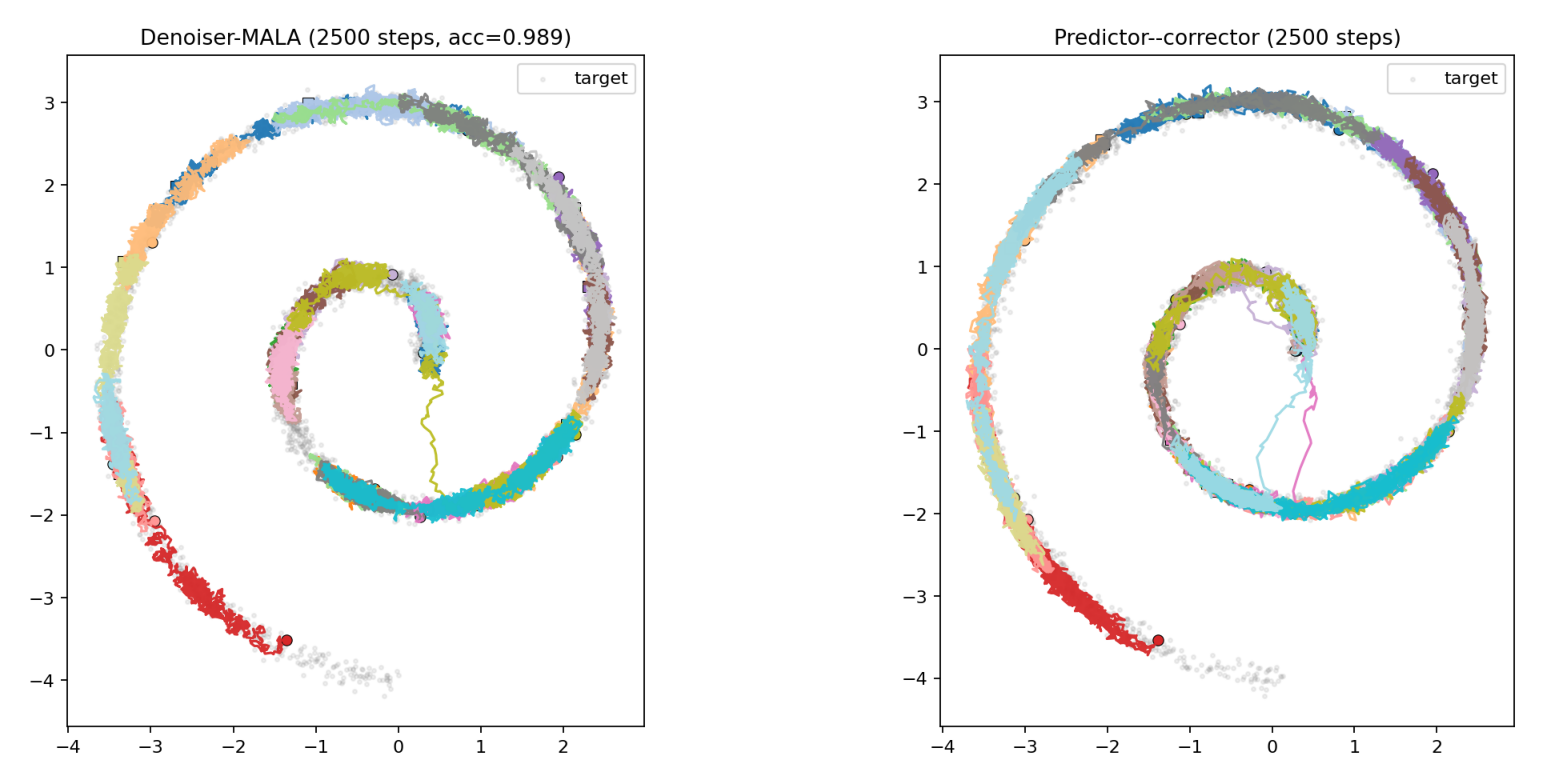}
    \caption{Side-by-side comparison. Left: dMALA trajectories on $p_\sigma$. Right: predictor--corrector flow on $p$. Both use 32 particles for 2500 steps. The figure demonstrates long chain behavior of both methods. Clearly, they continue to move broadly along the roll while remaining on the correct law.}
    \label{fig:swissroll_both_methods}
     \vspace{-1em}
\end{wrapfigure}

A key distinction between the two methods is the law they preserve. dMALA's dynamics is designed to preserve the smoothed distribution \(p_\sigma\), where the smoothing level is determined by the denoising time \(1-\tau = t_{\mathrm{noise}}\). In contrast, the predictor--corrector construction targets the original distribution \(p\) directly. For this reason, the quantitative evaluation is performed against \(p_\sigma\) for dMALA and against \(p\) for the predictor--corrector method.

%In all trajectory visualizations, we use \(32\) particles and run the chains for \(2500\) steps. We emphasize that these figures are intended to show long-run exploration of the manifold rather than short-term endpoint clouds. Since both samplers are initialized from samples already supported on the relevant distribution, the main qualitative question is whether they continue to move broadly along the roll while remaining on the correct law.

\textbf{Denoiser-Metropolis (dMALA).}
We first evaluate the corrected Langevin dynamics of \Cref{subsec:correction} using the denoiser extracted from the pretrained SoFlow model. We vary the denoising time \(t_{\mathrm{noise}}\), and therefore the implied smoothing level \(\sigma=(1-t_{\mathrm{noise}})/t_{\mathrm{noise}}\), while fixing the proposal step size to \(dt=\sigma^2\). Table~\ref{tab:swissroll_denoiser_mala} reports the resulting acceptance rate, the negative log-likelihood with respect to \(p_\sigma\), the MMD to samples drawn from \(p_\sigma\), and the mean displacement after \(200\) steps. 
The behavior of the dMALA step is observed in Table~\ref{tab:swissroll_denoiser_mala}.
% \begin{table}[t]
% \centering
% \caption{SoFlow-denoiser MALA ablation. This method preserves the smoothed law $p_\sigma$.}
% \label{tab:swissroll_denoiser_mala}
% \begin{tabular}{lcccccc}
% \toprule
% t_denoise & sigma & steps & Acc. rate & NLL to $p_\sigma$ & MMD to $p_\sigma$ & Mean move \\
% \midrule
% 0.90 & 0.111 & 200 & 0.896 & 2.912 & 0.016 & 1.508 \\
% 0.95 & 0.053 & 200 & 0.960 & 2.386 & 0.013 & 0.813 \\
% 0.98 & 0.020 & 200 & 0.989 & 1.733 & 0.017 & 0.370 \\
% \bottomrule
% \end{tabular}
% \end{table}

\begin{table}[t]
\small
\centering
\caption{SoFlow-denoiser MALA ablation. This method preserves the smoothed law $p_\sigma$.}
\label{tab:swissroll_denoiser_mala}
\begin{tabular}{lcccccc}
\toprule
$t_{\text{noise}}$ & $\sigma$ & steps & Acc. rate & NLL to $p_\sigma$ & MMD to $p_\sigma$ & Mean move \\
\midrule
0.90 & 0.111 & 200 & 0.896 & 2.912 & 0.016 & 1.508 \\
0.95 & 0.053 & 200 & 0.960 & 2.386 & 0.013 & 0.813 \\
0.98 & 0.020 & 200 & 0.989 & 1.733 & 0.017 & 0.370 \\
\bottomrule
\end{tabular}
\end{table}

\textbf{Predictor--corrector flow.}
We next evaluate the predictor--corrector mechanism of \Cref{subsec:predictor_corrector_flow}. Starting from \(x\sim p\), we partially noise the sample to the bridge distribution \(p_\tau\), and then use the learned SoFlow map to return it to \(p\). Table~\ref{tab:swissroll_predictor_corrector} reports results for several values of \(\tau\), including the negative log-likelihood with respect to the true target law \(p\), the MMD to samples drawn from \(p\), and the mean displacement after \(200\) steps. %Figure~\ref{fig:swissroll_predictor_corrector_traj} shows the corresponding long trajectory experiment for the best value of \(\tau\).

% \begin{table}[t]
% \small
% \centering
% \caption{Predictor--corrector flow ablation. This method targets the original law $p$.}
% \label{tab:swissroll_predictor_corrector}
% \begin{tabular}{lcccc}
% \toprule
% tau & steps & NLL to $p$ & MMD to $p$ & Mean move \\
% \midrule
% 0.70 & 200 & 8.994 & 0.057 & 2.811 \\
% 0.85 & 200 & 3.378 & 0.025 & 1.633 \\
% 0.95 & 200 & 1.760 & 0.020 & 0.568 \\
% \bottomrule
% \end{tabular}
% \end{table}
\begin{wraptable}{r}{0.6\textwidth} % 'r' for right alignment, 0.45\textwidth for width
\vspace{-1em}
\small
\centering
\caption{Predictor--corrector flow ablation. This method targets the original law $p$.}
\label{tab:swissroll_predictor_corrector}
\begin{tabular}{lcccc}
\toprule
tau & steps & NLL to $p$ & MMD to $p$ & Mean move \\
\midrule
0.70 & 200 & 8.994 & 0.057 & 2.811 \\
0.85 & 200 & 3.378 & 0.025 & 1.633 \\
0.95 & 200 & 1.760 & 0.020 & 0.568 \\
\bottomrule
\vspace{-2em}
\end{tabular}
\end{wraptable}

% \begin{figure}[t]
%     \centering
%     \includegraphics[width=0.8\linewidth]{swissroll_experiments/predictor_corrector/predictor_corrector_long_traj.png}
%     \caption{Swiss-roll trajectories for the predictor--corrector flow sampler. Starting from samples on the target distribution, the sampler first moves to the bridge distribution \(p_\tau\) and then returns to \(p\) through the learned SoFlow map. The plot shows \(32\) particles evolved for \(2500\) steps.}
%     \label{fig:swissroll_predictor_corrector_traj}
% \end{figure}

\textbf{Comparison of the two mechanisms.}
Finally, Figure~\ref{fig:swissroll_both_methods} compares the long-run trajectories of the two methods side by side using the same initial clean seeds. For the predictor--corrector method, these clean seeds are used directly. For dMALA, the same seeds are first perturbed by the corresponding Gaussian smoothing noise so that the initialization is consistent with the preserved law \(p_\sigma\). This comparison is intended to highlight the different exploration behavior of the two probability-preserving constructions. dMALA makes local accept/reject moves on the smoothed law, whereas the predictor--corrector method makes larger nonlocal moves by temporarily stepping away from the manifold and then returning to the target distribution through the learned flow.

% \begin{figure}[t]
%     \centering
%     \includegraphics[width=0.5\linewidth]{swissroll_experiments/both_methods/both_methods_long_traj.png}
%     \caption{Side-by-side comparison of the two probability-preserving mechanisms on the Swiss-roll target. Left: MALA trajectories on the smoothed law \(p_\sigma\). Right: predictor--corrector flow trajectories on the original target law \(p\). In both panels, \(32\) particles are evolved for \(2500\) steps.}
%     \label{fig:swissroll_both_methods}
% \end{figure}

Overall, the experiments are designed to illustrate the tradeoff between the two constructions. dMALA is a corrected local sampler with an explicit accept/reject mechanism and a well-defined invariant smoothed law \(p_\sigma\). The predictor--corrector method, by contrast, targets the original law \(p\) directly and can make larger moves by using the learned flow as a return map to the data distribution. The Swiss-roll experiments provide a controlled setting in which both behaviors can be observed clearly.

\begin{wrapfigure}{r}{0.45\textwidth}
    \centering
    \includegraphics[width=\linewidth]{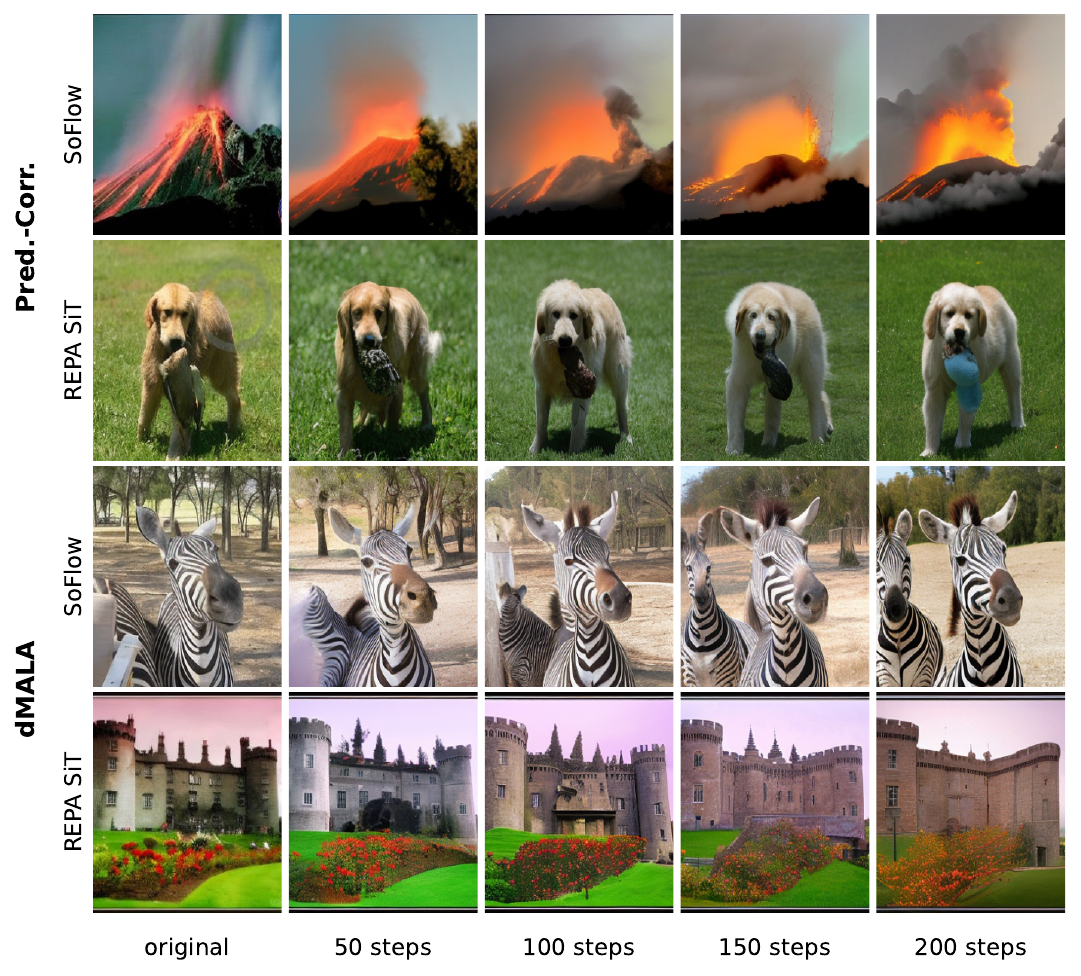}
    \caption{Visualization of ImageNet256 for two methods (Predictor--Corrector and dMALA) and two models (SoFlow-XL/2 and REPA SiT-XL/2). The first column shows the original image; each subsequent column shows the sample after an additional 50 iterations, up to 200 total. A high resolution version is provided in Appendix~\ref{app:high_resolution_imagenet}.}
    \label{fig:combined_imagenet}
\end{wrapfigure}

\subsection{ImageNet-256 generation}
\label{subsec:imnet_main}
We evaluate the two proposed stationary generation mechanisms on ImageNet-256 \citep{deng2009imagenet} using two pretrained latent flow backbones: SoFlow-XL/2-cond \citep{luo2025soflow} and SiT-REPA-XL \citep{yu2024representation}.
All metrics are computed on $50,000$ generated samples against the ADM
ImageNet-256 reference batch. For SoFlow, the baseline is the official
1-NFE solution-operator sampler. For SiT-REPA, the baseline is the no-CFG
Euler sampler used in our implementation. Our methods initialize from
data-supported VAE latents, apply a local noising step, and then return to the data distribution using either the predictor--corrector map or the
dMALA correction. The results of this experiment are shown in Table~\ref{tab:tab_imnet_combined}, which clearly shows that both methods give diverse enough images (Recall and IS scores) while maintaining very low FID.

% In our first experiment we test two flow models (SiTRepa and SoFlow) as engines for the denoiser and the corrector steps. 

\begin{table}[t]
\centering
\small
\setlength{\tabcolsep}{5pt}
\caption{Predictor--corrector and dMALA flows on ImageNet-256.}
\label{tab:tab_imnet_combined}
\begin{tabular}{@{}llccccc@{}}
\toprule
Backbone & Method / setting
& FID $\downarrow$ & sFID $\downarrow$ & IS $\uparrow$
& Precision $\uparrow$ & Recall $\uparrow$ \\
\midrule

\multicolumn{7}{l}{\textit{SoFlow-XL/2-cond backbone}} \\
& \textcolor{gray}{Baseline (1-NFE)}
& \textcolor{gray}{2.96} & \textcolor{gray}{4.60} & \textcolor{gray}{238.53}
& \textcolor{gray}{0.795} & \textcolor{gray}{0.564} \\
\cmidrule(lr){2-7}

& Predictor--Corrector, $K=10$, $t_{\rm noise}=0.3$
& 2.28 & 5.47 & 297.37 & 0.79 & 0.62 \\

& dMALA, $K=10$, $\sigma=0.3$
& \textbf{1.70} & \textbf{5.05} & 300.04 & 0.79 & \textbf{0.65} \\

& Predictor--Corrector, $K=20$, $t_{\rm noise}=0.5$
& 4.67 & 5.33 & \textbf{349.84} & \textbf{0.87} & 0.50 \\

& dMALA, $K=20$, $\sigma=0.3$
& 2.17 & 5.81 & 306.25 & 0.80 & 0.62 \\

\addlinespace[5pt]
\multicolumn{7}{l}{\textit{SiT-REPA-XL backbone}} \\
% & \textcolor{gray}{Baseline (no CFG, 25 Euler steps)}
% & \textcolor{gray}{11.38} & \textcolor{gray}{5.91} & \textcolor{gray}{118.30}
% & \textcolor{gray}{0.61} & \textcolor{gray}{0.70} \\
 & \textcolor{gray}{Baseline (no CFG, 400 Euler steps)}  & \textcolor{gray}{6.61} & \textcolor{gray}{5.62}  & \textcolor{gray}{152.52} & \textcolor{gray}{0.68} & \textcolor{gray}{0.69} \\
\cmidrule(lr){2-7}

& Predictor--Corrector, $K=10$, $t_{\rm noise}=0.3$
& 3.11 & 6.06 & 220.35 & 0.73 & 0.68 \\

& dMALA, $K=10$, $\sigma=0.3$
& 1.93 & 5.62 & \textbf{241.20} & \textbf{0.74} & 0.70 \\

& Predictor--Corrector, $K=20$, $t_{\rm noise}=0.5$
& 3.06 & 6.24 & 199.50 & 0.72 & 0.69 \\

& dMALA, $K=20$, $\sigma=0.3$
& \textbf{1.79} & \textbf{5.14} & 236.03 & \textbf{0.74} & \textbf{0.71} \\

\bottomrule
\end{tabular}
\end{table}

\textbf{Long Chain Generation:}
\label{subsec:long_chain}
In the next experiment we test how the methods behave for longer chains. To this end we increase $K$ (the number of steps in the chain and compute the same metrics. The results for dMALA and Predictor-corrector are presented in Table~\ref{tab:long_chain_soflow_combined}. The table shows only mild deterioration with the length of the chain. This is expected for approximate pretrained flows: although the ideal kernels are
probability-preserving, score approximation error by the pretrained model can accumulate
over repeated applications. Samples from this run can be observed in Figure~\ref{fig:combined_imagenet} and in high resolution in Figure~\ref{fig:combined_imagenet_app}.
%We therefore treat the number of probe steps as a
%controlled hyperparameter rather than claiming unconditional improvement for
%arbitrarily long chains. 

\textbf{Computational efficiency.}
Beyond chain length, the predictor--corrector mechanism is also
competitive in compute. With the SiT-REPA-XL backbone, $200$ model calls (number of function evaluations - NFEs)
attain FID $4.91$, lower than the $6.51$ FID of $800$-step
forward-Euler integration of the same backbone (\Cref{app:nfe_efficiency},
\Cref{tab:nfe_budget_sitrepa}).

We additionally evaluate our methods on the Oxford Flowers-102 dataset \citep{nilsback2008automated}. The quantitative results are detailed in \Cref{app:flowers_results}, and the corresponding visualizations are provided in \Cref{app:flowers}.

\begin{table}[t]
\centering
\small
\setlength{\tabcolsep}{4pt}
\caption{Long-chain behavior of dMALA and Predictor-Corrector on ImageNet-256
with SoFlow-XL/2-cond. dMALA uses
denoiser noise level $\sigma=0.3$; Predictor-Corrector uses bridge noise
$t_{\rm noise}=0.3$; $K$ denotes the number of iterations. Acc.\ rate is the
running mean Metropolis acceptance probability.}
\label{tab:long_chain_soflow_combined}
\begin{tabular}{@{}llccccccc@{}}
\toprule
Method & $K$ & Acc.\ rate & FID $\downarrow$ & sFID $\downarrow$
 & IS $\uparrow$ & Precision $\uparrow$ & Recall $\uparrow$ \\
\midrule
\textcolor{gray}{Baseline (1-NFE)} & \textcolor{gray}{--} & \textcolor{gray}{--}
 & \textcolor{gray}{2.96} & \textcolor{gray}{4.60} & \textcolor{gray}{238.53}
 & \textcolor{gray}{0.80} & \textcolor{gray}{0.56} \\
\cmidrule(lr){1-8}

dMALA & 50  & 0.76 & 3.23 & 7.72 & 313.29 & 0.81 & 0.57 \\
     & 100 & 0.54 & 4.34 & 10.40 & 308.11 & 0.82 & 0.52 \\
     & 150 & 0.53 & 5.07 & 12.45 & 301.23 & 0.81 & 0.49 \\
     & 200 & 0.53 & 5.61 & 14.26 & 294.47 & 0.81 & 0.47 \\
\midrule

Pred.--Corr. & 50  & --   & 3.96 & 8.90  & 313.48 & 0.82 & 0.53 \\
             & 100 & --   & 5.18 & 12.14 & 304.34 & 0.81 & 0.48 \\
             & 150 & --   & 5.93 & 14.64 & 290.25 & 0.80 & 0.46 \\
             & 200 & --   & 6.46 & 16.48 & 278.84 & 0.79 & 0.45 \\
\bottomrule
\end{tabular}
\end{table}
%%%%%%%%%%%%%%%%%%%%%%%%%%

% A similar result for predictor corrector are shown in Table~\ref{tab:long_chain_pred_cor} for the predictor corrector.

% \begin{figure}[t]
%     \centering
%     \includegraphics[width=0.65\linewidth]{Figures/imagenet_images_for_paper/combined_imagenet_all_methods_new.pdf}
%     \caption{Visualization of ImageNet256 for two methods
%              (Predictor--Corrector and dMALA) and two models (SoFlow-XL/2 and
%              REPA SiT-XL/2). The first column shows the original image;
%              each subsequent column shows the sample after an additional
%              50 iterations, up to 200 total. The high resolution image is provided in \Cref{app:high_resolution_imagenet}}
%     \label{fig:combined_imagenet}
% \end{figure}

\section{Discussion and Conclusion}
\label{sec:conclusion}

We have proposed a generative modeling framework in which the central
operation is not transport from noise to data, but approximate invariance of
the data law under a probability-preserving dynamics. Two mechanisms
instantiate this principle. The first is a corrected Langevin dynamics whose
Metropolis adjustment preserves a Gaussian-smoothed law \(p_\sigma\). The
second is a predictor--corrector flow that maps partially noised samples back
toward the original law \(p\). In the ideal-model setting, both define
invariant kernels; in practice, they can be implemented using quantities
available from pretrained flow-matching models, and therefore can be applied
on top of public checkpoints without retraining.
The broader impact of our method is
discussed in  \Cref{app:broader_impact}.

{\textbf{Limitations:}}
The effectiveness of both the Denoiser-Metropolis (dMALA) and Predictor--Corrector mechanisms is bound by the quality of the underlying pretrained flow or score-based model. Although the ideal kernels are probability-preserving in their respective ideal-model settings, approximation to the score  model can accumulate errors over repeated applications of the kernel. We explicitly characterize this behavior in \Cref{tab:long_chain_soflow_combined}. Our empirical validation is also restricted to continuous benchmarks (Swiss-roll, Oxford Flowers-102, ImageNet-256). Extending this paradigm to discrete domains such as text generation, where bridge models have only recently emerged, is left to future work.

The experiments support four main conclusions.
%First, on the synthetic
%Swiss-roll target, unadjusted Langevin dynamics visibly drifts away from the
%target distribution, while the corrected dynamics remains close to its
%preserved law (\Cref{fig:ULAMALA}); 
First, the quantitative results in
Tables~\ref{tab:swissroll_denoiser_mala}
and~\ref{tab:swissroll_predictor_corrector} show the same behavior for the
smoothed law \(p_\sigma\) and the original law \(p\), respectively. Second,
both samplers improve FID over the baseline sampler of every backbone tested.
%For example, with SoFlow-XL/2-cond, MALA at \(K=10\), \(\sigma=0.3\) attains
%FID \(1.70\), compared with \(2.96\) for the \(1\)-NFE baseline; with
%SiT-REPA-XL, MALA at \(K=20\), \(\sigma=0.3\) attains FID \(1.79\), compared
%with \(11.38\) for the \(25\)-step no-CFG Euler baseline
(\Cref{tab:tab_imnet_combined}). 
Third, the two samplers occupy different
operating regimes: dMALA with small smoothing gives the lowest FID, whereas
the predictor-corrector mechanism
 trades some FID for higher Inception Score and
Precision. This is consistent with the qualitative behavior in
\Cref{fig:swissroll_both_methods}: dMALA performs local accept/reject moves,
while predictor-corrector sampling makes larger non-local moves through the
learned flow. Finally, chain length and NFE budget provide graceful
trade-offs. Performance degrades smoothly on ImageNet-256,
consistent with accumulated approximation  errors in the
pretrained backbone (\Cref{tab:long_chain_soflow_combined}). For
velocity-field backbones, the predictor-corrector mechanism also compares
favorably with direct ODE integration: \(200\) NFEs of steps attain FID
\(4.91\), compared with \(6.51\) for \(800\)-step forward-Euler integration
(\Cref{tab:nfe_budget_sitrepa}).

% The limitations and broader impact of our method are
% discussed in Appendices~\ref{app:limitations}
% and~\ref{app:broader_impact}. 
To summarize, stationary probing offers an alternative to noise-to-data
transport: any sample from the
model target distribution already provides a valid initial state for the
probability-preserving dynamics.

\clearpage
\newpage

\bibliographystyle{unsrtnat}
\bibliography{references}

% \section*{References}

% References follow the acknowledgments in the camera-ready paper. Use unnumbered first-level heading for
% the references. Any choice of citation style is acceptable as long as you are
% consistent. It is permissible to reduce the font size to \verb+small+ (9 point)
% when listing the references.
% Note that the Reference section does not count towards the page limit.
% \medskip

% {
% \small

% [1] Alexander, J.A.\ \& Mozer, M.C.\ (1995) Template-based algorithms for
% connectionist rule extraction. In G.\ Tesauro, D.S.\ Touretzky and T.K.\ Leen
% (eds.), {\it Advances in Neural Information Processing Systems 7},
% pp.\ 609--616. Cambridge, MA: MIT Press.

% [2] Bower, J.M.\ \& Beeman, D.\ (1995) {\it The Book of GENESIS: Exploring
%   Realistic Neural Models with the GEneral NEural SImulation System.}  New York:
% TELOS/Springer--Verlag.

% [3] Hasselmo, M.E., Schnell, E.\ \& Barkai, E.\ (1995) Dynamics of learning and
% recall at excitatory recurrent synapses and cholinergic modulation in rat
% hippocampal region CA3. {\it Journal of Neuroscience} {\bf 15}(7):5249-5262.
% }

%%%%%%%%%%%%%%%%%%%%%%%%%%%%%%%%%%%%%%%%%%%%%%%%%%%%%%%%%%%%

\clearpage
\newpage

\appendix
\crefalias{section}{appendix}

%%%%%%%%%%%%%%%%%%%%%%%%%%%%%%%%%%%%%%%%%
\section*{Appendix Outline}

The appendices are organized as follows:

\begin{itemize}\setlength{\itemsep}{2pt}
    \item \Cref{app:related_work} discusses  related work relevant to our methods.
  \item \Cref{app:flow_to_denoiser} derives the denoiser used at sampling
        time from a pretrained flow-matching model, with a general formula
        for any linear bridge $x_t = \kappa(t)x + \sigma(t)z$ and an explicit
        specialization to the linear interpolant $x_t = tx + (1-t)z$.
  
  \item \Cref{app:algorithms} provides explicit pseudocode for one step of
        the Denoiser-Metropolis (dMALA) sampler (\Cref{alg:denoiser_metropolis_step})
        and the Predictor--Corrector sampler
        (\Cref{alg:predictor_corrector_step}), together with the cancellation
        argument that takes the trapezoidal line-integral acceptance ratio of
        \Cref{subsec:correction} into the simple two-evaluation form used in
        our experiments.
  \item \Cref{app:experimental_details} describes the datasets, pretrained
        backbones (SoFlow-XL/2-cond and SiT-REPA-XL), the Flowers-102 U-Net
        architecture, and the compute used.
  \item \Cref{app:additional_results} reports additional quantitative
        results: an evaluation on Oxford Flowers-102 and a dMALA
        ablation over smoothing levels $\sigma$ and chain lengths $K$ on
        ImageNet-256 with the SoFlow backbone, in addition to a study on the computational efficiency of our method.
  \item 
        \Cref{app:broader_impact} discusses broader societal impact.   
  \item \Cref{app:visualizations} shows generated-sample visualizations across
         step counts on Flowers-102 and ImageNet-256.

\end{itemize}

%==============================================================================

\section{Related Work}
\label{app:related_work}

\textbf{Generative Modeling.}
Deep generative models have progressed from autoregressive networks \citep{van2016conditional}, exact-likelihood normalizing flows \citep{dinh2016density, papamakarios2021normalizing}, and variational autoencoders \citep{kingma2013auto, rezende2014stochastic} to generative adversarial networks capable of high-resolution synthesis \citep{goodfellow2014generative, karras2019style}. More recently, denoising diffusion probabilistic models \citep{sohl2015deep, ho2020denoising, nichol2021improved} have achieved state-of-the-art sample quality by learning to reverse a gradual noising process. While these foundational models predominantly focus on mapping a simple noise prior to the target data distribution, our work provides a complementary framework that operates directly within the data distribution itself.

\textbf{Continuous-Time Transport and Flows.}
The generative processes of diffusion can be elegantly unified through the lens of continuous-time differential equations \citep{chen2018neural, grathwohl2018ffjord}. Recent formulations such as stochastic interpolants \citep{albergo2022building, albergo2023stochastic}, rectified flows \citep{liu2022flow}, and action matching \citep{neklyudov2023action} learn ordinary differential equations (ODEs) to smoothly transport probability mass from a base distribution to the target data. While traditionally optimized for efficient, straight-path generation from pure noise, the vector fields learned by these models provide the fundamental continuous-time dynamics that our predictor-corrector mechanism repurposes for local probability-preserving exploration.

\textbf{Data-Supported Initialization.}
Beyond unconditional generation from noise, initializing stochastic processes from data-supported states is widely used in inverse problems \citep{kawar2022denoising, wang2022zero} and guided image editing. Methods like SDEdit \citep{meng2021sdedit} inject controlled noise into an existing sample and reverse the process to produce realistic modifications, while deterministic inversion techniques map real data into latent spaces for localized semantic editing \citep{song2020denoising, hertz2022prompt, mokady2023null}. Although these methods primarily aim to balance user-guided constraints with sample realism, they share the underlying principle of exploring the local geometry around a given data manifold---a property we explicitly formalize as a stationary generative kernel.

\section{Estimating a Denoiser from a Flow Model}
\label{app:flow_to_denoiser}

A practical advantage of our framework is that it can reuse pretrained flow-matching models. In many settings, strong flow models are already available, and it is desirable to use them directly rather than train a separate denoiser. The key point is that the bridge used in flow matching already defines a consistent denoising problem near the data manifold. Thus a pretrained flow model can be converted into a denoiser in a way that is fully consistent with the interpolation on which it was trained.

Consider a general linear bridge
\begin{equation}
x_t = \kappa(t)\,x + \sigma(t)\,z,
\qquad
x\sim p,
\quad
z\sim \mathcal N(0,I),
\label{eq:generic_bridge_denoiser}
\end{equation}
where \(t\in[0,1]\), \(\kappa(1)=1\), and \(\sigma(1)=0\). To connect this bridge to denoising, define a noisy variable
\begin{equation}
x_\eta = x + \eta z.
\label{eq:noisy_variable_eta}
\end{equation}
If we choose
\begin{equation}
\eta = \frac{\sigma(t)}{\kappa(t)},
\label{eq:eta_from_bridge}
\end{equation}
then
\begin{equation}
\kappa(t)\,x_\eta
=
\kappa(t)\,x + \kappa(t)\eta z
=
\kappa(t)\,x + \sigma(t)\,z
=
x_t.
\label{eq:xt_equals_kappa_xeta}
\end{equation}
Thus the bridge variable \(x_t\) is simply a scaled version of the noisy sample \(x_\eta\).

Suppose first that the pretrained flow model provides a direct solution map
\begin{equation}
f_\theta(x_t,t,1) \approx x,
\label{eq:direct_solution_map_denoiser}
\end{equation}
that maps a point at time \(t\) to its prediction at time \(1\). Then, using Equation \eqref{eq:xt_equals_kappa_xeta}, the denoiser associated with noise level \(\eta\) is
\begin{equation}
D_\eta(y)
:=
\mathbb E[x \mid x_\eta = y]
\approx
f_\theta\!\bigl(\kappa(t)\,y,\; t,\; 1\bigr),
\qquad
\eta=\frac{\sigma(t)}{\kappa(t)}.
\label{eq:denoiser_from_solution_map_correct}
\end{equation}
This is the simplest and most useful form in practice: one evaluates the flow model at the scaled point \(\kappa(t)y\), not at \(y\) itself.

If instead the pretrained model provides only a velocity field
\begin{equation}
v_\theta(x_t,t) \approx \E[\dot x_t \mid x_t],
\label{eq:velocity_field_denoiser}
\end{equation}
then the denoiser can still be recovered. Differentiating Equation \eqref{eq:generic_bridge_denoiser} gives
\begin{equation}
\dot x_t = \kappa'(t)\,x + \sigma'(t)\,z.
\label{eq:generic_bridge_velocity_denoiser}
\end{equation}
Together, Equation \eqref{eq:generic_bridge_denoiser} and Equation \eqref{eq:generic_bridge_velocity_denoiser} form a \(2\times 2\) linear system in \(x\) and \(z\). Solving for \(x\) gives
\begin{equation}
x
=
\frac{\sigma'(t)\,x_t - \sigma(t)\,\dot x_t}
{\kappa(t)\sigma'(t)-\kappa'(t)\sigma(t)}.
\label{eq:x_from_xt_xtdot_correct}
\end{equation}
Substituting \(x_t=\kappa(t)y\) with \(y=x_\eta\) and taking conditional expectations yields
\begin{equation}
D_\eta(y)
\approx
\frac{\sigma'(t)\,\kappa(t)\,y - \sigma(t)\,v_\theta(\kappa(t)y,t)}
{\kappa(t)\sigma'(t)-\kappa'(t)\sigma(t)},
\qquad
\eta=\frac{\sigma(t)}{\kappa(t)}.
\label{eq:denoiser_from_velocity_correct}
\end{equation}
Again, the input to the flow model is the scaled point \(\kappa(t)y\).

For the linear bridge used in our experiments,
\begin{equation}
x_t = t\,x + (1-t)\,z,
\label{eq:linear_bridge_special_denoiser}
\end{equation}
we have \(\kappa(t)=t\), \(\sigma(t)=1-t\), and therefore
\begin{equation}
\eta = \frac{1-t}{t}.
\label{eq:eta_linear_bridge}
\end{equation}
In this case the denoiser induced by a direct solution map is simply
\begin{equation}
D_\eta(y)
\approx
f_\theta(t\,y,\; t,\; 1),
\qquad
t=\frac{1}{1+\eta}.
\label{eq:denoiser_from_solution_map_linear}
\end{equation}
If one uses the velocity field instead, then Equation \eqref{eq:denoiser_from_velocity_correct} reduces to
\begin{equation}
D_\eta(y)
\approx
t\,y + (1-t)\,v_\theta(t y,t),
\qquad
t=\frac{1}{1+\eta}.
\label{eq:denoiser_from_velocity_linear_correct}
\end{equation}

This construction is important for two reasons. First, it lets us use strong pretrained flow-matching codes directly as denoisers near the data manifold. Second, it keeps the method internally consistent: the same bridge used in training defines the denoiser used at sampling time. In particular, if \(t\) is chosen close to \(1\), then \(\eta\) is small, the denoiser acts only in a narrow neighborhood of the data manifold, and the resulting corrected sampler operates in exactly the regime on which the flow model was trained.

\FloatBarrier

%==============================================================================
\section{Algorithmic Details}
\label{app:algorithms}

This appendix gives explicit pseudocode for one step of each of the two probability-preserving samplers introduced in \Cref{sec:method}, and the short cancellation argument that turns the line-integral acceptance ratio of \Cref{subsec:correction} into the simple two-evaluation form actually used in our experiments.

\subsection{Denoiser-Metropolis (dMALA) Sampler}
\label{app:alg_denoiser_mala}

dMALA combines the proposal Equation \eqref{eq:denoiser_proposal_rewrite} with the simplified acceptance ratio Equation \eqref{eq:trap_acceptance_simplified_rewrite}. One step requires exactly two evaluations of the denoiser: one at the current state and one at the proposed state (Algorithm~\ref{alg:denoiser_metropolis_step}).

\begin{algorithm}[h]
\caption{One step of the Denoiser-Metropolis (dMALA) sampler.}
\label{alg:denoiser_metropolis_step}
\begin{algorithmic}[1]
\REQUIRE current state $x_j$, denoiser $D_\sigma$, smoothing level $\sigma > 0$
\STATE Draw $z \sim \mathcal{N}(0, I)$.
\STATE Form the proposal $y \;=\; D_\sigma(x_j) + \sqrt{2}\,\sigma\, z$. \hfill (cf.~Equation \eqref{eq:denoiser_proposal_rewrite})
\STATE Compute the approximate log-acceptance ratio
\[
\log r_{D_\sigma}(x_j, y)
\;\approx\;
\frac{1}{4\sigma^2}\Bigl(\bigl\lVert D_\sigma(x_j) - x_j\bigr\rVert^2 - \bigl\lVert D_\sigma(y) - y\bigr\rVert^2\Bigr).
\]
\hfill (cf.~Equation \eqref{eq:trap_acceptance_simplified_rewrite})
\STATE Set $\alpha \;=\; \min\bigl\{1,\; \exp(\log r_{D_\sigma}(x_j, y))\bigr\}$.
\STATE Draw $u \sim \mathrm{Unif}(0, 1)$.
\IF{$u \le \alpha$}
    \STATE Set $x_{j+1} = y$. \COMMENT{accept}
\ELSE
    \STATE Set $x_{j+1} = x_j$. \COMMENT{reject}
\ENDIF
\ENSURE $x_{j+1}$
\end{algorithmic}
\end{algorithm}

\subsection{Predictor--Corrector Sampler}
\label{app:alg_predictor_corrector}

The predictor--corrector sampler combines a single Gaussian noising step at
the bridge time $\tau$ with the application of the learned flow map
$\Phi_{\theta,\tau\to 1}$ back to the data end of the bridge
(Algorithm~\ref{alg:predictor_corrector_step}). Any procedure that
realizes this flow map suffices, and our framework is agnostic to the
choice. When the pretrained model exposes a direct solution map
$f_\theta(\cdot,\tau,1)$ as for SoFlow~\citep{luo2025soflow}, the corrector
is a single network function evaluation (NFE). When the model is given as
a velocity field $v_\theta$ as for
SiT-REPA-XL~\citep{yu2024representation}, the flow map is realized by
numerically integrating $\dot x = v_\theta(x,t)$ from $t=\tau$ to $t=1$
using any ODE integrator of choice; in our experiments we use forward
Euler with $n_{\mathrm{steps}}$ steps (Table~\ref{tab:nfe_budget_sitrepa}),
but Runge--Kutta variants, DPM-Solver, or any other integrator would
equally well realize the same flow map. The total cost per probe iteration
is therefore $1$ NFE for a direct solution map and $n_{\mathrm{steps}}$
NFEs for a numerically integrated velocity field.

\begin{algorithm}[h]
\caption{One step of the Predictor--Corrector sampler.}
\label{alg:predictor_corrector_step}
\begin{algorithmic}[1]
\REQUIRE current state $x_j$, bridge time $\tau \in (0,1)$, flow map $\Phi_{\theta,\tau\to 1}$
\STATE \textbf{Predictor:} draw $z \sim \mathcal{N}(0, I)$ and set
       $\hat{x} \;\gets\; (1-\tau)\,z + \tau\, x_j$.
       \hfill Equation \eqref{eq:pc_predictor_short}
\STATE \textbf{Corrector:} set $x_{j+1} \;\gets\; \Phi_{\theta,\tau\to 1}(\hat{x})$.
\ENSURE $x_{j+1}$
\end{algorithmic}
\end{algorithm}

\subsection{Derivation of the Trapezoidal Acceptance Ratio}
\label{app:trap_derivation}

We give the cancellation argument that takes the trapezoidal approximation of the line integral in Equation \eqref{eq:denoiser_trap_rewrite} into the simple two-evaluation form used in Equation \eqref{eq:trap_acceptance_simplified_rewrite}.

Recall that the corrected log-acceptance ratio is
\begin{equation}
\log r_{D_\sigma}(x, y) \;=\; \Delta_{D_\sigma}(x, y) + \log q_\sigma(x \mid y) - \log q_\sigma(y \mid x),
\label{eq:app_log_r_recall}
\end{equation}
with proposal density $q_\sigma(y \mid x) = \mathcal{N}\!\bigl(y;\, D_\sigma(x),\, 2\sigma^2 I\bigr)$. The Gaussian proposal log-ratio is
\begin{equation}
\log q_\sigma(x \mid y) - \log q_\sigma(y \mid x)
\;=\;
\frac{1}{4\sigma^2}\Bigl(\bigl\lVert y - D_\sigma(x)\bigr\rVert^2 - \bigl\lVert x - D_\sigma(y)\bigr\rVert^2\Bigr).
\label{eq:app_q_ratio}
\end{equation}
Writing $r_\sigma(x) := D_\sigma(x) - x$, expand each squared norm:
\begin{align}
\lVert y - D_\sigma(x)\rVert^2 &= \lVert (y - x) - r_\sigma(x)\rVert^2 \;=\; \lVert y-x\rVert^2 - 2\,\langle y-x,\, r_\sigma(x)\rangle + \lVert r_\sigma(x)\rVert^2, \\
\lVert x - D_\sigma(y)\rVert^2 &= \lVert (x - y) - r_\sigma(y)\rVert^2 \;=\; \lVert y-x\rVert^2 - 2\,\langle x-y,\, r_\sigma(y)\rangle + \lVert r_\sigma(y)\rVert^2.
\end{align}
Subtracting and using $\langle x-y,\,r_\sigma(y)\rangle = -\langle y-x,\,r_\sigma(y)\rangle$,
\begin{equation}
\lVert y - D_\sigma(x)\rVert^2 - \lVert x - D_\sigma(y)\rVert^2
\;=\;
-2\,\langle y - x,\, r_\sigma(x) + r_\sigma(y)\rangle + \lVert r_\sigma(x)\rVert^2 - \lVert r_\sigma(y)\rVert^2.
\label{eq:app_norm_diff}
\end{equation}
Substituting Equation \eqref{eq:app_norm_diff} into Equation \eqref{eq:app_q_ratio} and combining with the trapezoidal approximation
\(
\Delta_{D_\sigma}(x, y) \approx \tfrac{1}{2\sigma^2}\,\langle r_\sigma(x) + r_\sigma(y),\, y - x\rangle,
\)
the inner-product terms cancel and only the squared-residual difference survives:
\begin{equation}
\log r_{D_\sigma}(x, y) \;\approx\; \frac{1}{4\sigma^2}\Bigl(\lVert D_\sigma(x) - x\rVert^2 - \lVert D_\sigma(y) - y\rVert^2\Bigr),
\end{equation}
which is Equation \eqref{eq:trap_acceptance_simplified_rewrite}. The sampler therefore needs only two evaluations of $D_\sigma$ per step, one at the current state and one at the proposal.

\FloatBarrier

%==============================================================================
\section{Experimental Details}
\label{app:experimental_details}

\subsection{Datasets}

\textbf{Oxford Flowers-102.}
We use the Oxford 102 Category Flower dataset~\citep{nilsback2008automated}, which comprises 8{,}189 images spanning 102 fine-grained flower species. All three official splits (train, validation, and test) are merged into a single evaluation pool, following prior unconditional generation work. All images are resized to $256 \times 256$ pixels.

\textbf{ImageNet-256.}
For large-scale experiments we use ImageNet 2012~\citep{deng2009imagenet} at $256 \times 256$ resolution with 1{,}000 classes. Generation operates in the latent space of a variational autoencoder~\citep{rombach2022high}, following the convention of recent latent generative models such as REPA~\citep{yu2024representation}.

\subsection{Generative Models}

All three backbones share the same variational autoencoder \citep{rombach2022high}, a KL-regularized encoder--decoder trained to produce 4-channel latents at $1/8$ the spatial resolution of the input image (i.e.\ $32 \times 32$ for $256 \times 256$ inputs). Latents are scaled by $\lambda = 0.18215$ before and after the generative model, following the convention of the Stable Diffusion codebase.

\subsubsection{Oxford Flowers-102}
\label{sec:flowers_unet}

For the Flowers-102 experiments we train a latent flow matching model from scratch using a U-Net~\citep{ronneberger2015u} that operates directly on the 4-channel latent tensors. The architecture is summarized in Table~\ref{tab:unet_arch}. The model is trained with a cosine-path interpolant
\begin{equation}
  x_t = \cos\!\left(\tfrac{\pi}{2}t\right)\,x + \sin\!\left(\tfrac{\pi}{2}t\right)\,\varepsilon,
  \qquad x \sim p,\; \varepsilon \sim \mathcal{N}(0,I),
  \label{eq:cosine_bridge}
\end{equation}
so that $t=0$ corresponds to pure data and $t=1$ to pure noise. Generation from noise integrates the learned velocity backward from $t=1$ to $t=0$ using the second-order Runge--Kutta (RK2) integrator~\citep{butcher2016numerical}.

\begin{table}[h]
\centering
\caption{Flow Matching U-Net architecture (Flowers-102).}
\label{tab:unet_arch}
\begin{tabular}{ll}
\toprule
Hyperparameter & Value \\
\midrule
Latent channels          & 4 \\
Latent resolution        & $32 \times 32$ \\
Residual blocks per level & 3 \\
Channel multipliers      & $[1,\,2,\,2,\,2]$ \\
Attention heads          & 4 \\
Attention resolutions    & $16 \times 16$ \\
Dropout                  & 0.1 \\
Interpolant              & Cosine path \\
ODE integrator           & RK2, $N_{\mathrm{ode}}=10$ \\
\bottomrule
\end{tabular}
\end{table}

\subsubsection{ImageNet-256}
\label{sec:soflow}

SoFlow~\citep{luo2025soflow} is a class-conditional latent flow matching model built on the DiT-XL/2 backbone~\citep{peebles2023scalable}. We use the publicly available checkpoint trained for $1.2 \times 10^6$ gradient steps with the EMA model weights. The model uses a linear interpolant
\begin{equation}
  x_t = (1-t)\,z + t\,x, \qquad t \in [0,1],
\end{equation}
with the Euler coefficient type, operating on $32 \times 32$ 4-channel latents and conditioning on 1{,}000 ImageNet class labels.

We additionally use the SiT-XL/2 checkpoint trained with REPA~\citep{yu2024representation}, a training objective that aligns a DiT's internal representations with a frozen vision encoder. We generate $50{,}000$ samples and evaluate using the REPA evaluator script. For both backbones, all generations reported in this paper are produced without classifier-free guidance.

\subsection{Implementation Details}
\label{sec:impl_details}

All experiments for Oxford Flowers-102 ran on a single NVIDIA RTX 4090 GPU. ImageNet-256 experiments used two NVIDIA RTX 4090 GPUs and one H100 GPU. We did not retrain or finetune any of the pretrained ImageNet backbones (SoFlow-XL/2-cond, SiT-REPA-XL); the 50{,}000-sample evaluation runs are the only ImageNet-scale compute consumed by our framework.

\FloatBarrier

%==============================================================================
\section{Additional Experimental Results}
\label{app:additional_results}

\subsection{Oxford Flowers-102}
\label{app:flowers_results}

We evaluate our approach on the Oxford Flower-102 dataset~\citep{nilsback2008automated} (training details are in \ref{sec:flowers_unet}). Table~\ref{tab:fm_probe_metrics_flowers} compares the Predictor--Corrector and dMALA samplers against a standard flow-matching baseline generated from Gaussian noise, reporting Frechet Inception Distance (FID), Precision, and Recall across step counts. 

\begin{table}[h]
\centering
\caption{FID, Precision, and Recall on Oxford Flowers-102, with $t_{\text{noise}}=0.1$.}
\label{tab:fm_probe_metrics_flowers}
\begin{tabular}{lcccc}
\toprule
\textbf{Method} & \textbf{Steps} &  \textbf{FID} $\downarrow$ & \textbf{Precision} $\uparrow$ & \textbf{Recall} $\uparrow$ \\
\midrule
\textcolor{gray}{Baseline FM (noise $\to$ data)} & \textcolor{gray}{--} & \textcolor{gray}{23.95} & \textcolor{gray}{0.5213} & \textcolor{gray}{0.2761} \\
\midrule
Predictor--Corrector & 20 & 12.50 & 0.6013 & 0.6962 \\
                     & 40 & 15.90 & 0.4676 & 0.6313 \\
                     & 60 & 19.92 & 0.3859 & 0.5848 \\
                     & 80 & 25.38 & 0.3091 & 0.5524 \\
\midrule
dMALA  & 20 & \textbf{10.92} & \textbf{0.6967} & \textbf{0.7825} \\
 & 40 & 12.64 & 0.5940 & 0.6936 \\
                      & 60 & 14.45 & 0.5179 & 0.6537 \\
& 80 & 16.34 & 0.4587 &  0.6210 \\
                     
\bottomrule
\end{tabular}
\end{table}

\FloatBarrier

\subsection{ImageNet-256: dMALA Ablation}
\label{app:dm_soflow_ablation}

Table~\ref{tab:dm_soflow} reports the dMALA sampler on ImageNet-256 with the SoFlow-XL/2-cond backbone, sweeping the smoothing level $\sigma$ at $\sigma \in \{0.3,\,0.5\}$ and the number of Metropolis--Hastings iterations $K$. Smaller $\sigma$ raises the acceptance rate and yields lower FID; larger $\sigma$ trades FID against higher Inception Score and Precision, at the cost of Recall. 

\begin{table}[h]
\centering
\small
\setlength{\tabcolsep}{4pt}
\caption{dMALA on ImageNet-256 with SoFlow-XL/2-cond. The denoiser uses SoFlow's solution operator with noise level $\sigma$. $K$ denotes the number of Metropolis--Hastings iterations; Acc.\ rate is the running mean acceptance probability. Metrics are computed on $50{,}000$ samples against the ADM reference batch. Reference baseline shown in gray.}
\label{tab:dm_soflow}
\begin{tabular}{@{}lccccccc@{}}
\toprule
 & $K$ & Acc.\ rate & FID $\downarrow$ & sFID $\downarrow$ & IS $\uparrow$
 & Precision $\uparrow$ & Recall $\uparrow$ \\
\midrule
\textcolor{gray}{Baseline (1-NFE)}
 & \textcolor{gray}{--} & \textcolor{gray}{--} & \textcolor{gray}{2.96} & \textcolor{gray}{4.60} & \textcolor{gray}{238.53} & \textcolor{gray}{0.80} & \textcolor{gray}{0.56} \\
\addlinespace[4pt]
$\sigma = 0.5$
 & 50  & 0.68 & 5.50 & 6.57  & 356.05 & 0.87 & 0.47 \\
 & 100 & 0.49 & 7.28 & 8.18  & 367.45 & 0.89 & 0.42 \\
 & 150 & 0.49 & 8.29 & 9.49  & 372.80 & 0.89 & 0.39 \\
 & 200 & 0.48 & 9.02 & 10.32 & 377.30 & 0.90 & 0.38 \\
\addlinespace[4pt]
$\sigma = 0.3$
 & 50  & 0.76 & 3.23 & 7.72  & 313.29 & 0.81 & 0.57 \\
 & 100 & 0.54 & 4.34 & 10.40 & 308.11 & 0.82 & 0.52 \\
 & 150 & 0.53 & 5.07 & 12.45 & 301.23 & 0.81 & 0.49 \\
 & 200 & 0.53 & 5.61 & 14.26 & 294.47 & 0.81 & 0.47 \\
\bottomrule
\end{tabular}
\end{table}

\FloatBarrier

\subsection{Computational Efficiency}
\label{app:nfe_efficiency}

In the next experiment we test the computational efficiency of the methods and compare them for a transport generation that has a similar budget. Results are shown in Table~\ref{tab:nfe_budget_sitrepa}.
\begin{table*}[h]
\centering
\small
\setlength{\tabcolsep}{4pt}
\caption{Predictor--corrector sampling matches or beats direct flow
integration at equal or smaller compute on ImageNet-256 with
SiT-REPA-XL. The baseline rows integrate the learned velocity field
$v_\theta$ from $t{=}1$ to $t{=}0$ using forward Euler with the
indicated number of steps. Our method runs the
Algorithm~\ref{alg:predictor_corrector_step} sampler for $K$
predictor--corrector iterations; the corrector inside each iteration
integrates $v_\theta$ from $t{=}\tau$ to $t{=}1$ using forward Euler
with $n_{\text{steps}}$ steps, so the total network function
evaluations satisfy $\mathrm{NFE}=n_{\text{steps}}\cdot K$. All our
methods use $t_{\text{noise}}{=}0.5$. Metrics are computed on
$50{,}000$ samples against the ADM reference batch.}
\label{tab:nfe_budget_sitrepa}
\begin{tabular}{@{}lcccccc@{}}
\toprule
 & NFE & FID $\downarrow$ & sFID $\downarrow$ & IS $\uparrow$
 & Precision $\uparrow$ & Recall $\uparrow$ \\
\midrule
%\multicolumn{7}{l}{\textit{Baseline budget: 400 Euler steps}} \\
  \textcolor{gray}{Baseline (400 Euler steps)} & \textcolor{gray}{400} & \textcolor{gray}{6.61} & \textcolor{gray}{5.62}  & \textcolor{gray}{152.52} & \textcolor{gray}{0.68} & \textcolor{gray}{0.69} \\
  \textcolor{gray}{Baseline (800 Euler steps)}
  & \textcolor{gray}{800} & \textcolor{gray}{6.51} & \textcolor{gray}{5.58} & \textcolor{gray}{153.08} & \textcolor{gray}{0.68} & \textcolor{gray}{0.69} \\
\hline
Pred.--Corr $n_{\text{steps}}{=}10$, $K{=}20$ (ours)
  & 200 & \textbf{4.91} & 12.07 & 183.82 & 0.70 & 0.66 \\
Pred.--Corr $n_{\text{steps}}{=}20$, $K{=}20$ (ours)
  & 400 & \textbf{3.06} & 6.24  & 199.50 & 0.72 & 0.69 \\
\bottomrule
\end{tabular}
\end{table*}

%==============================================================================
% \section{Limitations}
% \label{app:limitations}

% The effectiveness of both the Denoiser-Metropolis (dMALA) and Predictor--Corrector mechanisms is bound by the quality of the underlying pretrained flow or score-based model. Although the ideal kernels are probability-preserving in their respective ideal-model settings, approximation to the score  model can accumulate errors over repeated applications of the kernel. We explicitly characterize this behavior in \Cref{tab:long_chain_soflow_combined}. Our empirical validation is also restricted to continuous benchmarks (Swiss-roll, Oxford Flowers-102, ImageNet-256). Extending this paradigm to discrete domains such as text generation, where bridge models have only recently emerged, is left to future work.

%==============================================================================
\section{Broader Impact}
\label{app:broader_impact}

This work develops a sampling framework that operates on top of existing pretrained flow-matching models. Two implications follow. On the positive side, because our framework re-uses existing checkpoints rather than training new generative models, it can reduce the energy and compute footprint of deploying generative systems, and pairs naturally with retrieval-augmented pipelines that aim to ground generated content in reliable source material. On the negative side, because the framework inherits the behavior of the underlying pretrained model, it also inherits whatever biases, representational gaps, or potential for misuse the underlying model carries, including the risk of being applied to produce deepfakes or targeted misinformation. Practitioners deploying this framework should treat the underlying model's documentation and known failure modes as a baseline that the sampling framework does not modify.

%==============================================================================

\section{Sample Visualizations}
\label{app:visualizations}

We provide visualizations of both the Predictor--Corrector and the dMALA samplers on the ImageNet256 dataset \citep{deng2009imagenet}. \Cref{fig:combined_imagenet_app} presents a combined high-resolution overview comparing both methods across the SoFlow-XL/2 and SiT-XL/2 models over 200 iterations. 

To examine the short-term behavior of the samplers, \Cref{fig:denoiser_20steps} and \Cref{fig:probe_20steps} illustrate the original images alongside samples generated after a 20-step run using dMALA and Predictor--Corrector, respectively. 

Furthermore, we provide extended long-run visualizations for Predictor--Corrector in \Cref{fig:app_probe_200} and for dMALA in \Cref{fig:app_denoiser_200}. In these extended subfigures, the first row displays the original image (before any iterations), and each subsequent row corresponds to an additional 50 steps, reaching 200 steps in the final row. Across all experiments, both methods successfully preserve the class identity throughout the chain while progressively moving the sample along the data manifold.
To further illustrate the step-by-step progression of the Predictor--Corrector across longer chains, \Cref{fig:imagenet_fm_repa} and \Cref{fig:imagenet_fm_soflow} present generated samples using the SiT-REPA-XL and SoFlow-XL/2-cond backbones, respectively. In these figures, the leftmost column presents the original image, while the subsequent columns track the continuous evolution of the state across steps.

\subsection{ImageNet-256}
\label{app:high_resolution_imagenet}

\begin{figure}[h]
    \centering
    \includegraphics[width=\linewidth]{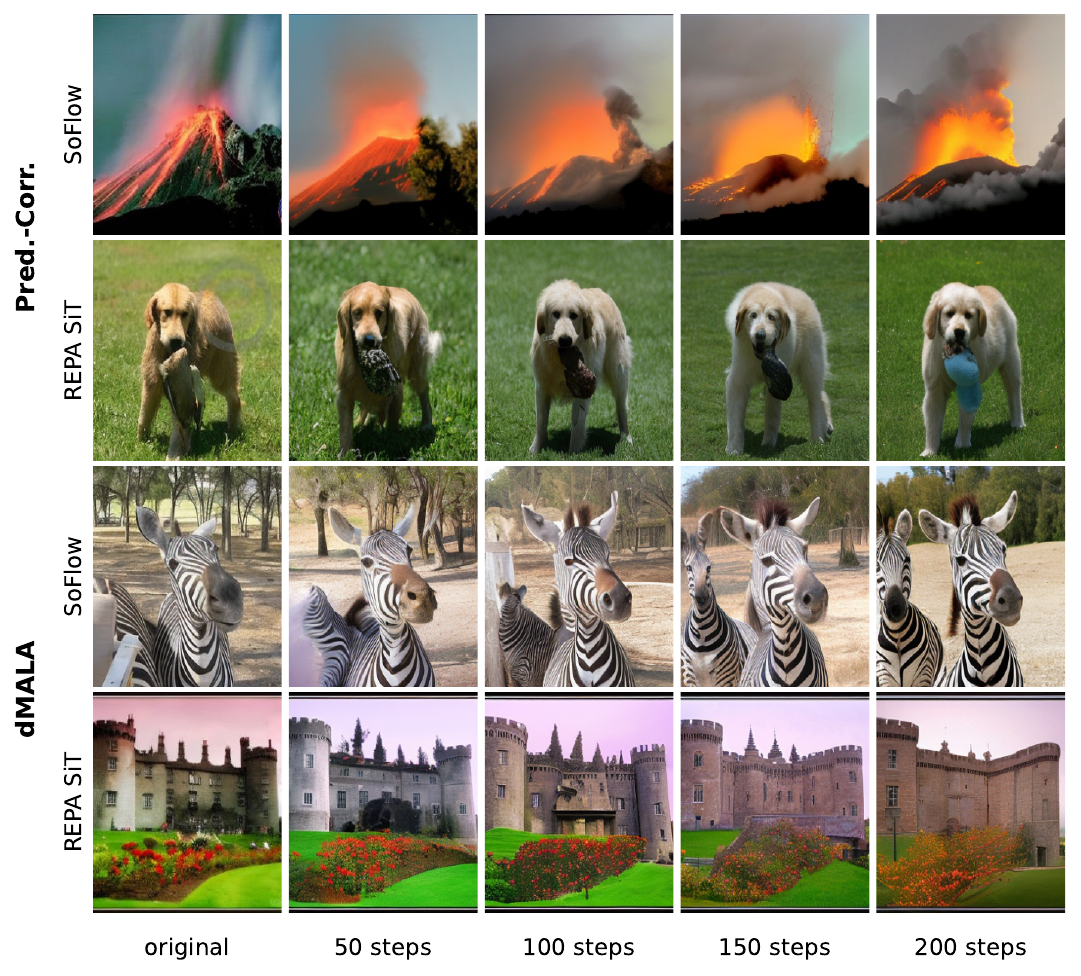}
    \caption{Visualization of ImageNet256 for two methods
             (Predictor--Corrector and dMALA) and two models (SoFlow-XL/2 and
             REPA SiT-XL/2). The first column shows the original image;
             each subsequent column shows the sample after an additional
             50 iterations, up to 200 total.}
    \label{fig:combined_imagenet_app}
\end{figure}

% ── Denoiser MALA — 20 steps ──────────────────────────────────────────────────
\begin{figure}[h]
    \centering
    \begin{subfigure}[h]{0.48\textwidth}
        \centering
        \includegraphics[width=\textwidth]{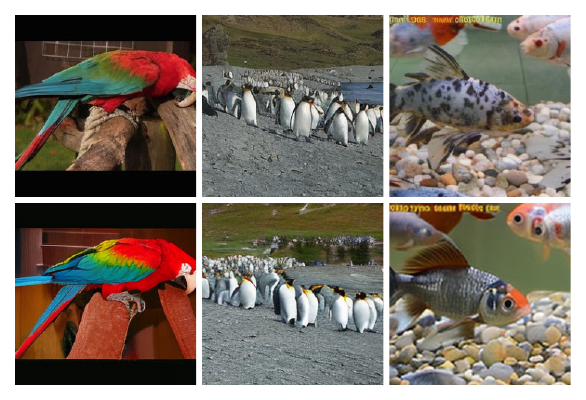}
        \caption{SoFlow-XL/2 }
    \end{subfigure}
    \hfill
    \begin{subfigure}[h]{0.48\textwidth}
        \centering
        \includegraphics[width=\textwidth]{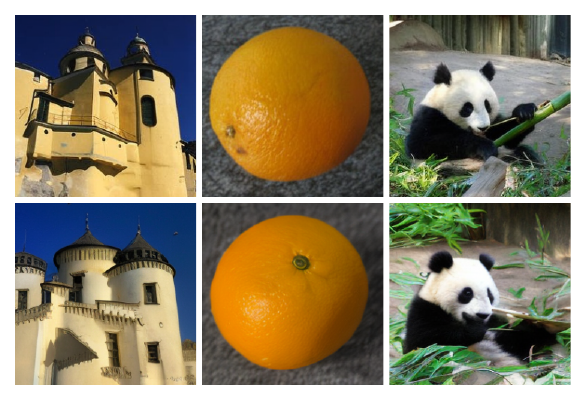}
        \caption{SiT-XL/2 }
    \end{subfigure}
    \caption{dMALA (20 steps). The first row shows the original image;
             the second row shows the generated sample.}
    \label{fig:denoiser_20steps}
\end{figure}

% ── FM Probe — 20 steps ───────────────────────────────────────────────────────
\begin{figure}[h]
    \centering
    \begin{subfigure}[t]{0.48\textwidth}
        \centering
        \includegraphics[width=\textwidth]{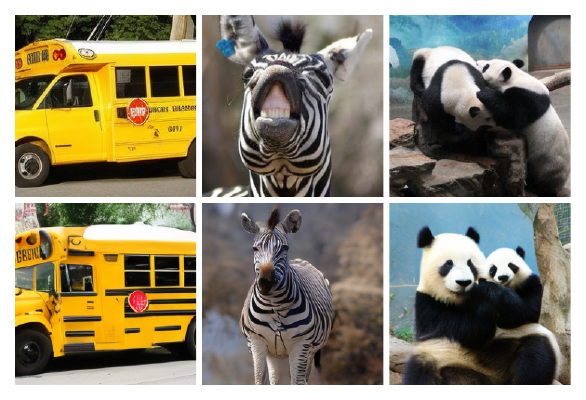}
        \caption{SoFlow-XL/2}
    \end{subfigure}
    \hfill
    \begin{subfigure}[t]{0.48\textwidth}
        \centering
        \includegraphics[width=\textwidth]{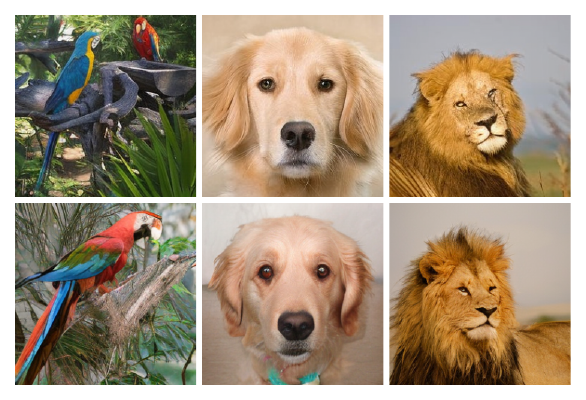}
        \caption{SiT-XL/2}
    \end{subfigure}
    \caption{Pred.--Corr. (20 steps). The first row shows the original image;
             the second row shows the generated sample.}
    \label{fig:probe_20steps}
\end{figure}

% We provide extended visualizations of both the Predictor--Corrector and the dMALA
% sampler run for long runs on ImageNet256 \citep{deng2009imagenet}.
% Figures~\ref{fig:app_probe_200} and~\ref{fig:app_denoiser_200} show results
% for SoFlow-XL/2 and SiT-XL/2 respectively.
% In each subfigure the first row is the original image (before any iterations),
% and each subsequent row corresponds to an additional 50 steps,
% reaching 200 steps in the final row.
% Both methods preserve the class identity throughout the chain while
% progressively moving the sample along the data manifold.

% ── FM Probe — 200 steps ─────────────────────────────────────────────────────
\begin{figure}[h]
    \centering
    \begin{subfigure}[t]{0.48\textwidth}
        \centering
        \includegraphics[width=\textwidth]{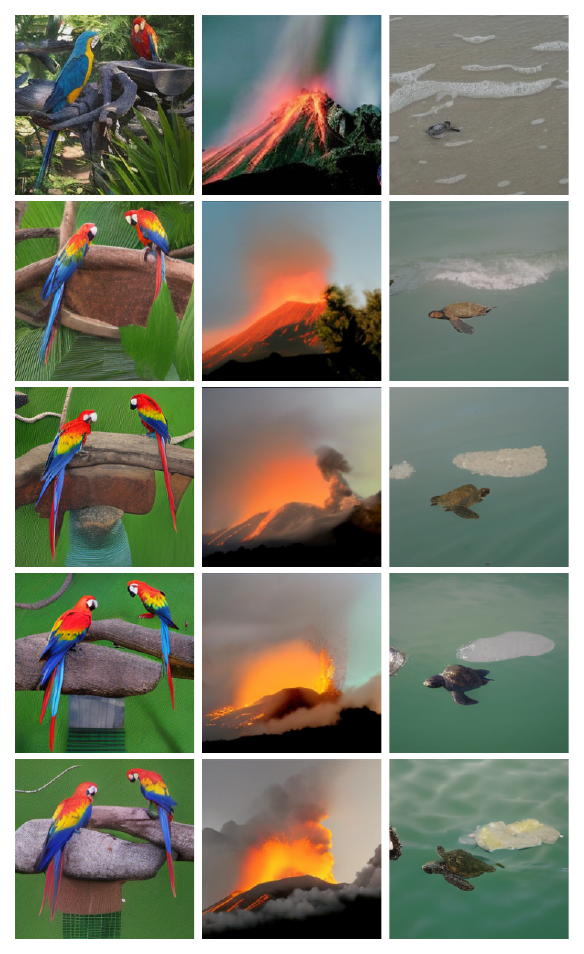}
        \caption{SoFlow-XL/2 }
    \end{subfigure}
    \hfill
    \begin{subfigure}[t]{0.48\textwidth}
        \centering
        \includegraphics[width=\textwidth]{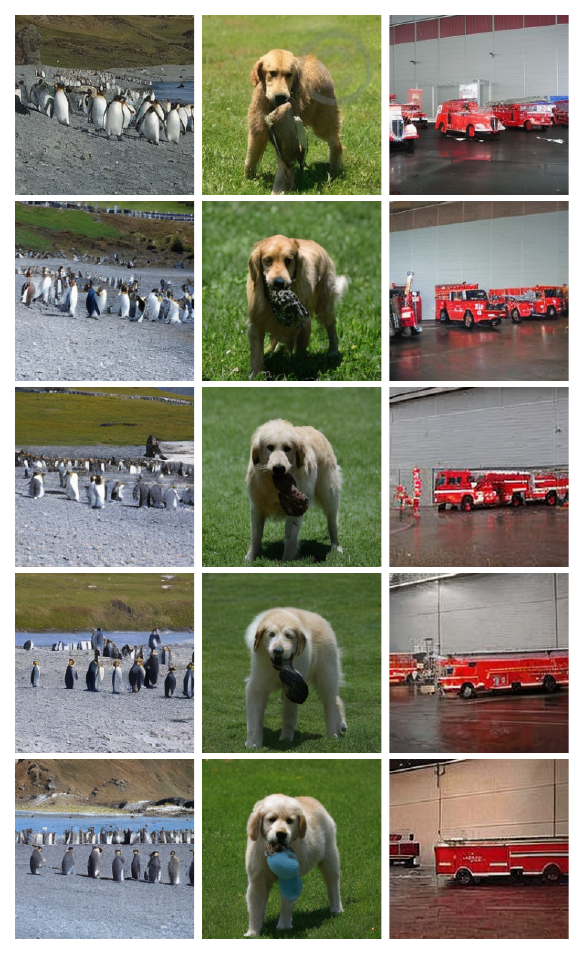}
        \caption{SiT-XL/2 }
    \end{subfigure}
    \caption{Predictor--Corrector on ImageNet256 over 200 iterations.
             The first row shows the original image; each subsequent row shows
             the sample after an additional 50 steps, up to a total of 200 steps.}
    \label{fig:app_probe_200}
\end{figure}

% ── Denoiser MALA — 200 steps ────────────────────────────────────────────────
\begin{figure}[h]
    \centering
    \begin{subfigure}[t]{0.48\textwidth}
        \centering
        \includegraphics[width=\textwidth]{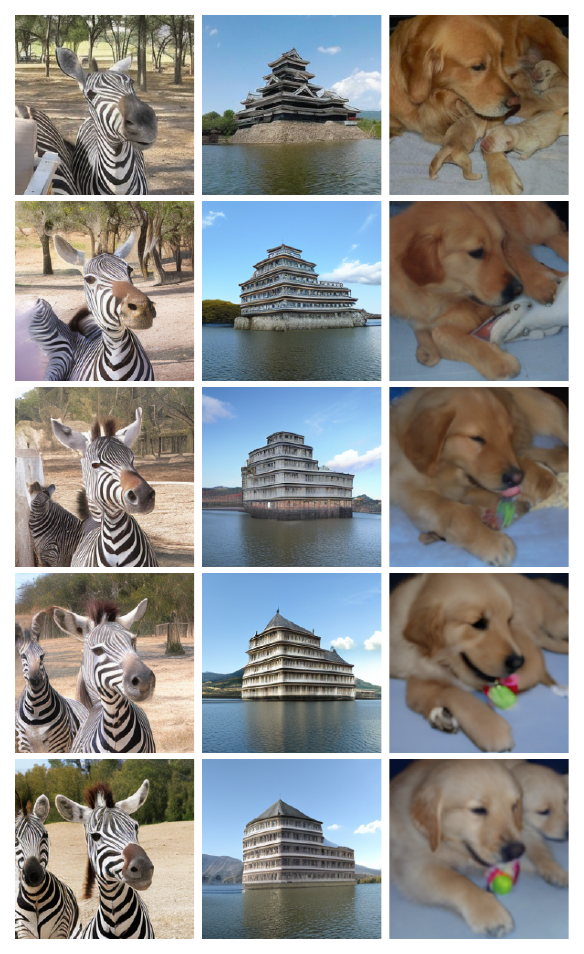}
        \caption{SoFlow-XL/2}
    \end{subfigure}
    \hfill
    \begin{subfigure}[t]{0.48\textwidth}
        \centering
        \includegraphics[width=\textwidth]{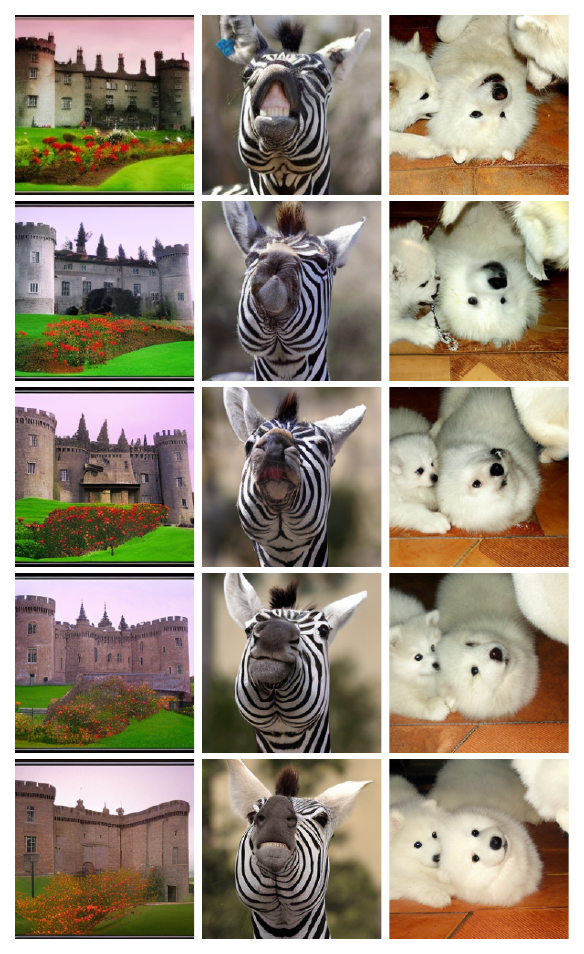}
        \caption{SiT-XL/2}
    \end{subfigure}
    \caption{dMALA on ImageNet256 over 200 iterations.
             The first row shows the original image; each subsequent row shows
             the sample after an additional 50 steps, up to a total of 200 steps.}
    \label{fig:app_denoiser_200}
\end{figure}

\begin{figure}[h]
    \centering
    \includegraphics[width=\textwidth]{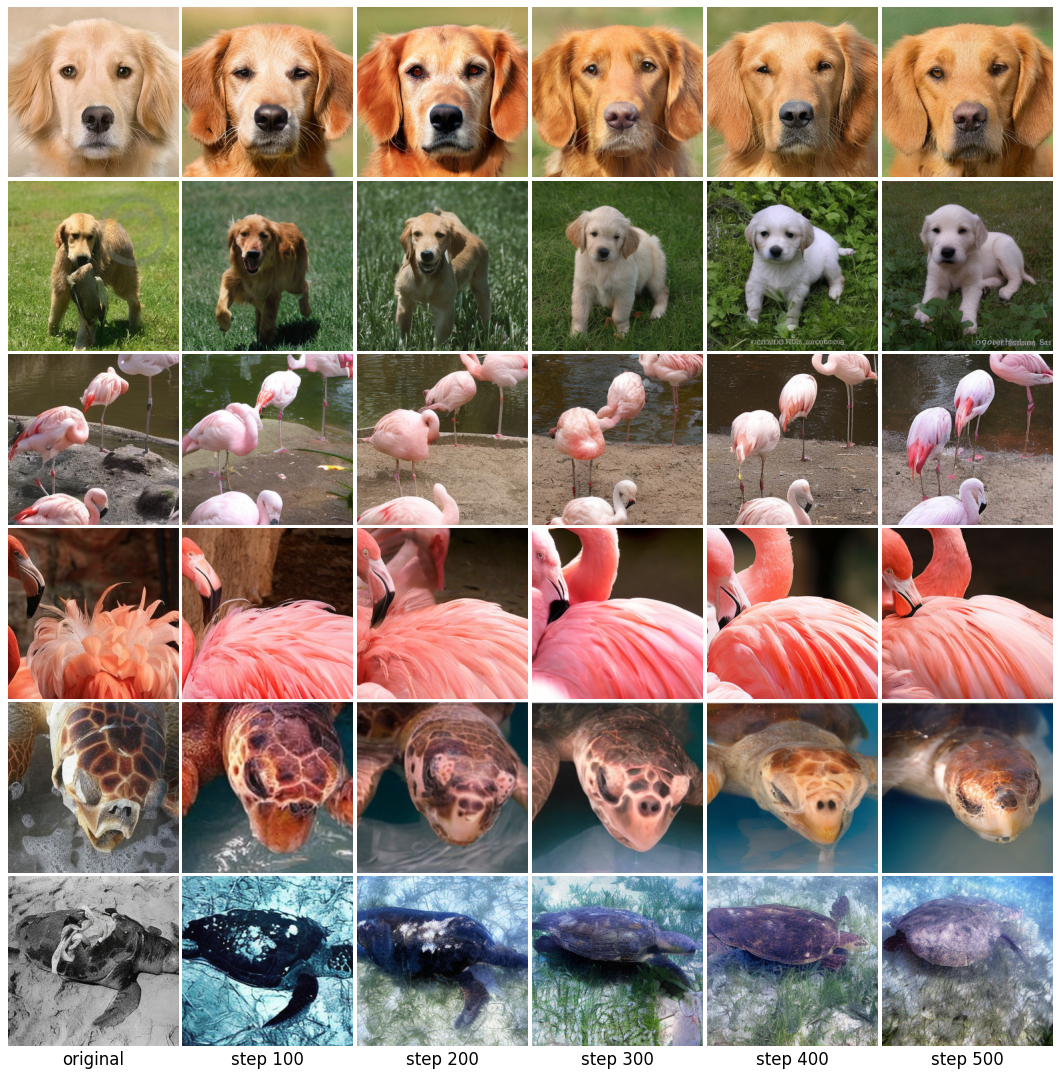}
    \caption{Pred.-Corr. ImageNet-256 samples generated by with the SiT-REPA-XL backbone. The leftmost column shows the original image; subsequent columns show the state after the indicated number of probe iterations.}
    \label{fig:imagenet_fm_repa}
\end{figure}

\begin{figure}[h]
    \centering
    \includegraphics[width=\textwidth]{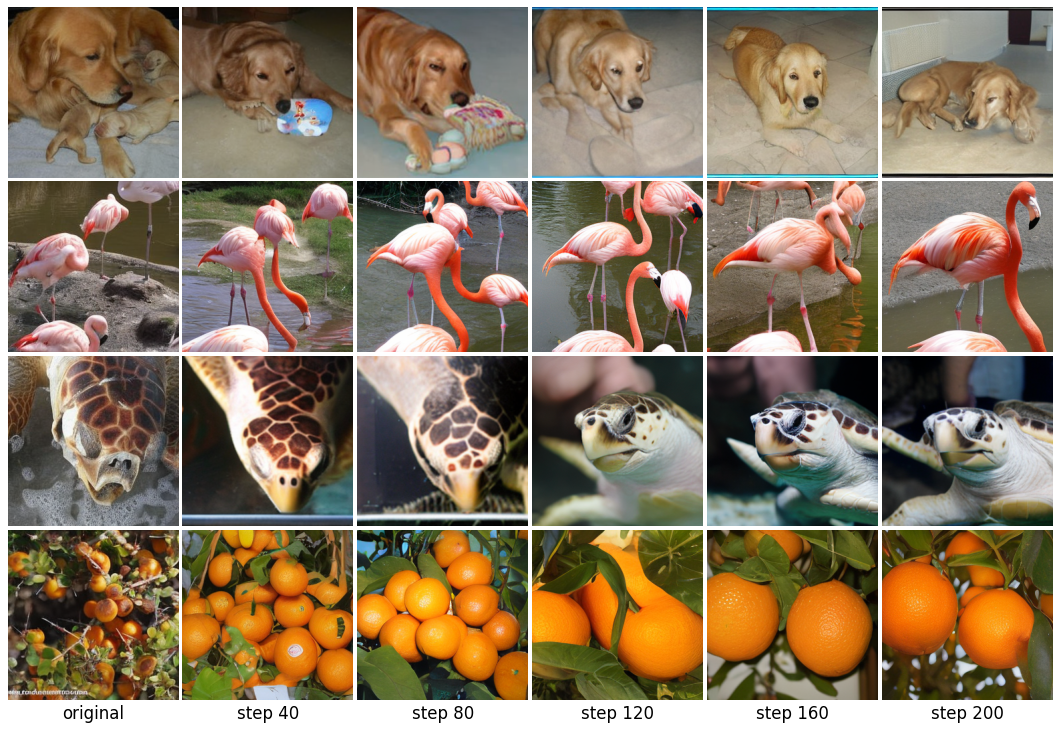}
    \caption{Pred.-Corr. ImageNet-256 samples generated by with the SoFlow-XL/2-cond backbone. The leftmost column shows the original image; subsequent columns show the state after the indicated number of probe iterations.}
    \label{fig:imagenet_fm_soflow}
\end{figure}

\subsection{Oxford Flowers-102 Visualization}
\label{app:flowers}

We further validate our approach on the Oxford Flowers dataset \citep{nilsback2008automated}.
Figure~\ref{fig:app_flowers_probe} 
shows results for the Predictor--Corrector and dMALA.
In each subfigure the first row shows the original image, and each
subsequent row shows the sample after an additional 20 iterations.

\begin{figure}[h]
    \centering
    \begin{subfigure}[t]{0.48\textwidth}
        \centering
        \includegraphics[width=\textwidth]{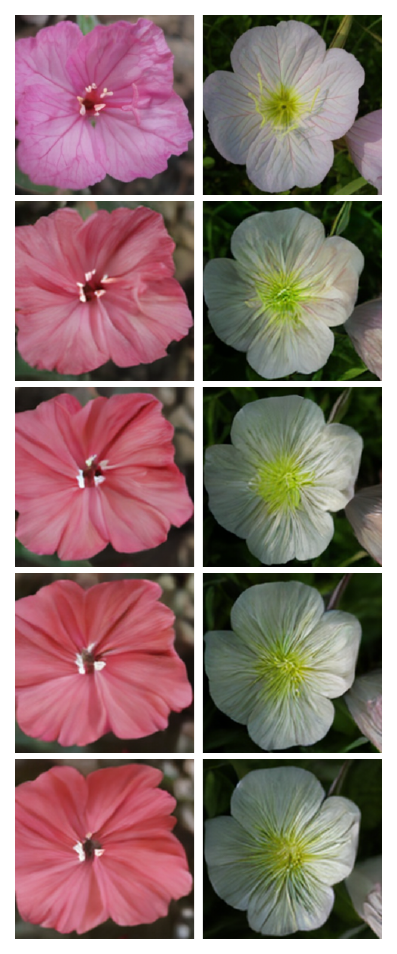}
        \caption{Predictor--Corrector}
    \end{subfigure}
    \hfill
    \begin{subfigure}[t]{0.48\textwidth}
        \centering
        \includegraphics[width=\textwidth]{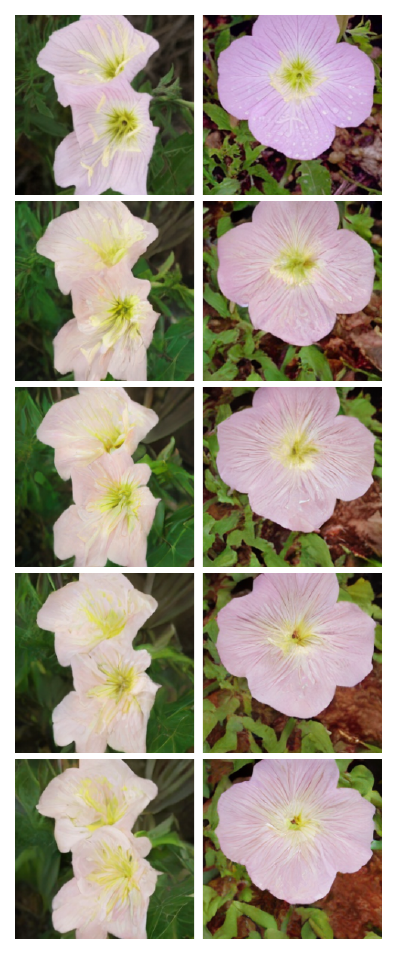}
        \caption{dMALA}
    \end{subfigure}
    \caption{Predictor--Corrector and dMALA on Oxford Flowers-102 over 80 iterations.
             The first row shows the original image; each subsequent row shows
             the sample after an additional 20 iterations, up to a total of 80 steps.}
    \label{fig:app_flowers_probe}
\end{figure}

\end{document}